\documentclass[10pt]{article} 

\usepackage[accepted]{rlj} 

\usepackage{amssymb}            
\usepackage{mathtools}          
\usepackage{mathrsfs}           
\usepackage{graphicx}           
\usepackage{subcaption}         
\usepackage[space]{grffile}     
\usepackage{url}                
\usepackage{amsthm}
\usepackage{booktabs}
\usepackage{multirow}
\newtheorem{theorem}{Theorem}
\newtheorem{assumption}{Assumption}
\newtheorem{corollary}{Corollary}
\newtheorem{remark}{Remark}
\newtheorem{lemma}{Lemma}
\newtheorem{example}{Example}

\newif\ifshowrevisions
\showrevisionsfalse

\newcommand{\rev}[1]{%
  \ifshowrevisions
    \textcolor{blue}{#1}%
  \else
    #1%
  \fi
}

\newenvironment{revision}
{\ifshowrevisions\color{blue}\fi}
{}

\title{Escaping Offline Pessimism: Vector-Field Reward Shaping for Safe Frontier Exploration}

\setrunningtitle{Escaping Offline Pessimism: Vector-Field Reward Shaping for Safe Frontier Exploration}

\author{
Amirhossein Roknilamouki\textsuperscript{1},
Arnob Ghosh\textsuperscript{2},
Eylem Ekici\textsuperscript{1},
Ness B. Shroff\textsuperscript{1,3}
}

\emails{
roknilamouki.1@osu.edu,
arnob.ghosh@njit.edu,
ekici.2@osu.edu,
shroff.11@osu.edu
}

\affiliations{
$^{1}$\textbf{Department of Electrical and Computer Engineering,
The Ohio State University, Columbus, OH, USA}\\

$^{2}$\textbf{Department of Electrical and Computer Engineering,
New Jersey Institute of Technology, Newark, NJ, USA}\\

$^{3}$\textbf{Department of Computer Science and Engineering,
The Ohio State University, Columbus, OH, USA}
}

\contribution{
    \rev{This paper studies how to train an RL policy offline such that, when
    deployed online, it jointly optimizes its primary-task reward and safely gathers informative data near the boundary of known regions by allocating stationary state-visitation mass between task-relevant
    and exploratory regions.}
    }
    {
    Unlike safe RL methods that rely on iterative online policy updates and explicit safety certification during interaction, our setting focuses on stationary deployment with no online policy adaptation.
    }
    \contribution{

    Assuming access to a differentiable uncertainty oracle, we propose a vector-field exploratory reward combining a gradient-alignment term that attracts the agent toward a target uncertainty level with a rotational-flow
term that promotes motion tangent to the corresponding uncertainty manifold.(Figure~\ref{fig:exp_a_vf_method_particles}).
    }
    {
    As illustrated in Figure~\ref{fig:exp_a_bump_method_particles}, without the rotational-flow  component, the agent may exhibit degenerate behavior and achieve only limited exploration.
    }

    \contribution{
    We provide theoretical analysis showing that the proposed reward shaping concentrates behavior near the target uncertainty manifold while encouraging tangential mobility along the boundary.
    }
    {
   The analysis explains why the proposed reward encourages the desired boundary exploration behavior, but it does not provide a full convergence guarantee for deep RL training with function approximation.
    }
    
  \contribution{
    We evaluate the proposed reward on a 2D navigation task with regions where the simulator is unreliable relative to the real world, and show that when combined with Soft Actor-Critic, the learned policy approaches these regions and explores their boundary while avoiding transitions that move deeply into the unreliable region.
    }
    {
    \rev{ The experiment is intended as a proof-of-concept demonstration of the
    induced behavior rather than a broad empirical benchmark across tasks or
    algorithms. The uncertainty oracle is specified analytically in this
    experiment; practical estimators and their calibration are discussed in the
    paper.}
    }


\summary{
Offline reinforcement learning often produces conservative policies whose
pessimism limits online exploration and data collection. Inspired by safe
reinforcement learning, we target the boundary of regions well covered by
offline data and reliably modeled by the simulator, where samples are
informative but uncertainty remains moderate. However, naively
rewarding this boundary-seeking behavior can lead to a degenerate ``parking''
behavior, where the agent simply stops once it reaches the frontier. To solve
this, we propose a novel vector-field reward-shaping paradigm designed to induce
continuous boundary exploration for fixed deployed policies. Assuming access to
a differentiable uncertainty oracle, our reward combines two complementary
components: a gradient-alignment term that attracts the agent toward a target
uncertainty level, and a rotational-flow term that promotes motion along the
local tangent plane of the uncertainty manifold. Our average-reward analysis
shows that the policy jointly optimizes the primary MDP reward and exploratory
motion, with the trade-off characterized through its stationary state-visitation
distribution. Finally, we validate this sustained, goal-directed exploration through a 2D continuous navigation task.
}

\begin{document}

\maketitle  

\begin{abstract}
Offline reinforcement learning often produces conservative policies whose
pessimism limits online exploration and data collection. Inspired by safe
reinforcement learning, we target the boundary of regions well covered by
offline data and reliably modeled by the simulator, where samples are
informative but uncertainty remains moderate. However, naively rewarding this
boundary-seeking behavior during offline training can lead to a degenerate
parking behavior at deployment, where the fixed policy stops once it reaches
the frontier. To address this issue, we propose a novel vector-field
reward-shaping paradigm designed to induce continuous boundary exploration for
fixed policies at deployment. Assuming access to a differentiable uncertainty
oracle, our reward combines two complementary components: a gradient-alignment
term that attracts the agent toward a target uncertainty level, and a
rotational-flow term that promotes motion along the local tangent plane of the
uncertainty manifold. Our average-reward analysis shows that the policy jointly
optimizes the primary MDP reward and exploratory motion, with the trade-off
characterized through its stationary state-visitation distribution. In a 2D continuous-navigation task with an analytically specified uncertainty
oracle, integrating the reward with Soft Actor-Critic induces sustained
boundary exploration while retaining goal-directed behavior.
\end{abstract}



\section{Introduction}
\label{sec:introduction}
Deep reinforcement learning (RL) has achieved remarkable success in simulated domains, yet training agents from scratch in real-world, safety-critical environments—such as autonomous driving or robotics—remains prohibitively dangerous. Unconstrained online exploration frequently leads to catastrophic failures, such as high-speed collisions or unrecoverable system states. To mitigate these risks, offline RL provides a powerful paradigm: learning a reliable policy entirely from previously collected datasets prior to deployment \citep{levine2020offline}. However, learning solely from static data introduces the severe challenge of out-of-distribution (OOD) actions. To prevent the agent from executing unseen and potentially hazardous behaviors, offline methods typically enforce strict pessimism, either by heavily penalizing uncertain actions \citep{kumar2020conservative} or by rigidly confining the learned policy to the support of the training data \citep{kostrikov2021offline}.

While this pessimism is crucial for ensuring initial safety, it introduces a paradoxical challenge for subsequent improvement. \rev{A natural progression in RL is to deploy the offline-trained agent to collect
targeted online data for later offline fine-tuning.} Yet, the very pessimism that ensures safety severely paralyzes the agent's ability to explore novel states during this online phase \citep{mark2022fine}. Conversely, completely overriding this pessimism to force standard exploration immediately reintroduces the risk of real-world catastrophes. Safe RL literature offers a conceptual bridge to this paradox. By incorporating cost constraints alongside traditional reward signals, Safe RL restricts exploration to a recoverable, safe neighborhood around the initial policy's coverage \citep{roknilamouki2025provably,pacchiano2025contextual, shi2023near,  ghosh2022provably, pacchiano2021stochastic}. Much like a newly licensed driver who cautiously tests a car's handling limits without pushing into dangerous extremes, a safe RL agent seeks out "reasonable risks"—regions of the state space where the agent is uncertain, but the potential for harm remains strictly bounded and recoverable.

\rev{Applying this safe RL logic to modern deep RL, however, presents a
fundamental challenge. To manage risk while expanding coverage, the agent must
iteratively update its estimate of which states remain novel or uncertain and
favor states with moderate, recoverable uncertainty, using newly collected
samples to refine both its uncertainty model and policy. Without these updates,
a fixed novelty-seeking policy may degenerate by repeatedly visiting only a few
states that were initially attractive. In structured settings, traditional safe
RL methods can certify that successive exploration decisions or policy updates
remain within prescribed safety bounds using tractable analyses for tabular,
linear, or Gaussian-process models \citep{berkenkamp2023bayesian}. Extending
this type of certification to deep uncertainty estimators and neural policies
is considerably more difficult: online gradient updates can alter deployment
behavior in unpredictable ways, and establishing their reliability generally
requires extensive offline validation before deployment
\citep{huang2024non}.}

\rev{To overcome these limitations, we propose learning and encoding the desired
exploratory behavior entirely during offline training, rather than relying on
risky online updates to the agent's uncertainty estimate or policy. The goal is
to obtain a fixed policy that, upon deployment, performs the primary task while
also collecting informative data near the boundary of regions well covered by
the offline dataset or simulator. By keeping this exploration close to familiar
regions, the agent can take manageable risks without venturing deeply into
uncertain, potentially unrecoverable states. Because this training occurs
entirely offline, researchers can thoroughly test the safety of the resulting
behavior and adjust the trade-off between primary-task performance and
exploration before real-world deployment.}

\rev{This approach is well suited to real-world deployment because the policy
parameters remain frozen during online interaction. A fixed policy avoids
unpredictable online updates while retaining the exploratory behavior learned
offline, allowing the agent to collect informative data near the boundary of
familiar regions. These data can then be used in a later offline phase to
enrich datasets, improve simulators, or train better reward models. Although
recent work has studied theoretical frameworks for fixed-policy
offline-to-online data collection \citep{huang2024non, zhang2023policy}, how to
learn such behavior in scalable deep RL policies remains an open challenge.
This motivates the central question of our paper:}


\begin{center}
\rev{
\textit{How can we train a policy entirely offline such that, when deployed
online without adaptation, it balances primary-task performance with safe,
informative data collection near the boundary of established coverage?}}
\end{center}

However, designing a reward signal that reliably induces this behavior is nontrivial. A natural idea is to encourage the agent to move toward regions of higher uncertainty by assigning larger rewards to states near the boundary of the known region. Intuitively, such a reward landscape concentrates incentives near this frontier, where informative samples can be collected while remaining close to recoverable states. However, purely state-based intrinsic rewards of this form often lead to degenerate behaviors. Because the reward is maximized at particular boundary points, the agent may simply ``park'' at a single high-reward location rather than traverse the entire frontier. As a result, the state visitations collapse into a small set of modes instead of spreading along the boundary (See Figure~\ref{fig:exp_a_bump_method_particles}). Similar concentration effects have been observed in state-visitation matching objectives, where directly maximizing a target density can lead to mode collapse \citep{lee2019efficient}. Consequently, inducing sustained exploration along a boundary requires a reward mechanism that not only attracts the agent toward the frontier but also promotes motion along it.

To solve this, we introduce a novel \textbf{vector-field exploratory reward} designed specifically to induce safe, continuous boundary exploration. \rev{Assuming access to
a differentiable uncertainty oracle, our reward comprises two complementary
elements:}

\textbf{1. Gradient Alignment:} A near-potential-based component that encourages the agent to follow the uncertainty gradient toward the target boundary level.

 \textbf{2. Rotational Flow:} A novel intrinsic signal that projects the agent's movement along the local tangent plane of the uncertainty manifold. Once the agent reaches the target boundary, this rotational element prevents degenerate ``parking'' and pushes the agent to continuously traverse the safe frontier.

Through theoretical analysis, we prove that this vector-field reward naturally induces the desired safe exploratory behavior. We evaluate our method on a 2D navigation task containing regions where the simulator is unreliable relative to the real world. When combined with Soft Actor-Critic (SAC) \citep{haarnoja2018soft}, the resulting policy exhibits the desired safe exploratory behavior: the agent approaches these regions and explores their boundary while avoiding transitions that move deeply into the unreliable region.

\textbf{Related Work.} Our work touches upon several areas of exploration and safe RL, but departs from them in crucial ways to support scalable, stationary deployment:
\begin{itemize}
\item \textbf{State-Visitation and Density Matching:}
A closely related line of work encourages exploration by matching the agent's state-visitation distribution to a target distribution. Methods such as State Marginal Matching (SMM) \citep{lee2019efficient} formulate this objective as a two-player game and have shown promising results in guiding exploration. However, achieving accurate coverage typically requires maintaining an online estimate of the state marginal distribution and, in practice, sampling from a collection of previously trained policies. We refer the reader to Remark~\ref{remark:smm_challenges} for a discussion of the relationship between this framework and our approach. Moreover, another related work in this category is \citet{hazan2019provably}.
\item \textbf{Informative Sampling for Offline RL:} Works such as \citet{zhang2023policy, huang2024non} focus on designing non-adaptive policies from offline data to gather informative samples. However, these methods are largely analytical, and focus on tabular MDPs. 
\item \textbf{Reward-Free and Sim-to-Real Exploration:} Methods exploring reward-free learning \citep{wagenmaker2022reward} or offline-to-online exploration \citep{wagenmaker2024overcoming}  assume linear function approximation. Moreover, they do not penalize the exploratory policy for entering unsafe regions.
\end{itemize}

\section{Problem Formulation}

We consider a true, real-world Markov Decision Process (MDP) defined by the tuple $\mathcal{M}^* = (\mathcal{S},\mathcal{A},P^*,r^*)$, where the state space $\mathcal{S}\subset \mathbb{R}^d$ is compact, and the reward function is bounded such that $|r^*(s,a)|\le R_{\max}$.   Prior to online deployment, we train the agent offline within a simulated MDP, $\widehat{\mathcal{M}} = (\mathcal{S}, \mathcal{A}, \widehat{P}, \widehat{r})$. \rev{For any stationary policy $\pi$, let $\mu_\pi$ denote its stationary
state-visitation distribution under $\widehat{P}$, and let $\mu_\pi(B)$ denote the probability mass it assigns to any measurable set
$B\subseteq\mathcal{S}$.} Due to finite data coverage and modeling approximations, there exists an
inevitable sim-to-real gap between the simulated and true environments.  To quantify this gap, let $\Delta(s)$ capture the true discrepancy between the simulator's dynamics and reward functions and those of the real world at a given state $s$. We formulate this as the worst-case divergence over the action space:
\begin{equation}
\label{eq:uncertainty_oracle}
    \Delta(s) := \sup_{a\in\mathcal{A}} \left[ d\left(\widehat{P}(\cdot|s,a), P^*(\cdot|s,a)\right) + \lambda \left| \widehat{r}(s,a) - r^*(s,a) \right| \right],
\end{equation}
where $d(\cdot, \cdot)$ denotes a suitable probability metric (e.g., Total Variation distance or Kullback-Leibler divergence) \citep{lattimore2020bandit}, and $\lambda > 0$ is a weighting coefficient.

\rev{\textbf{Practical Construction and Calibration of the Uncertainty Oracle.}
Because $P^*$ and $r^*$ are unknown during offline training, the true
discrepancy $\Delta(s)$ cannot be computed directly. We therefore assume access
to a twice continuously differentiable uncertainty oracle
$U:\mathcal{S}\to\mathbb{R}$ that is conservatively calibrated such that
$\Delta(s)\leq U(s)$ for all $s\in\mathcal{S}$, as in recent
offline-to-online RL analyses \citep{uehara2405bridging}. In practice, $U$ may
be estimated from epistemic uncertainty or data coverage. For example, a
Gaussian Process fitted to state data yields predictive variance that is low
near observed states and increases away from them; in high-dimensional
settings, last-layer features of a deep neural network can be used to define
the kernel \citep{uehara2405bridging}. Random Network Distillation provides a
scalable, differentiable novelty score \citep{burda2018exploration}, while
implicit density estimation through discriminative models offers another
practical proxy \citep{fu2017ex2, bellemare2016unifying}. These methods do not automatically provide
the conservative upper bound assumed in our analysis and therefore require
task-specific calibration, for example using held-out model errors and an
additional safety margin. In our experiments, we specify $U$ analytically as a
smooth Gaussian-shaped uncertainty landscape to isolate the behavior induced
by the proposed reward shaping from uncertainty-estimation error.}


During online deployment, unconstrained exploration can push the fixed policy into high-uncertainty, out-of-distribution (OOD) regions where the sim-to-real gap is unreliably large \citep{zhouefficient, nakamoto2023cal, kumar2020conservative}. Conversely, strictly confining the agent to low-uncertainty regions halts data collection. Inspired by safe reinforcement learning, which advocates taking manageable risks to expand the exploration frontier \citep{kitamura2025provably, roknilamouki2025provably, shi2023near}, we propose a controlled exploration strategy targeting a specific uncertainty level, $U_{\mathrm{mid}} \in \mathbb{R}$. This encourages the agent to visit states that balance risk: novel enough to yield informative data, yet close enough to known regions to ensure safe recovery. We formalize this target exploration manifold as the level set: $\mathcal{U} := \{s \in \mathcal{S} : U(s) = U_{\mathrm{mid}}\}$. Practically, $U_{\mathrm{mid}}$ is a user-defined risk tolerance hyperparameter, calibrated strictly below the uncertainty threshold where the sim-to-real gap causes unrecoverable performance degradation.

\section{Proposed Approach (Vector Field-Guided Reward Shaping)}
\label{sec:approach}
Having established the target manifold $\mathcal{U}$, our objective is to train a policy within the offline simulator such that, upon deployment in the online phase, it simultaneously \rev{balances primary-task performance with informative exploration along this
boundary.} To build intuition for the problem setting and the desired deployment behavior, we consider the following illustrative example.

\begin{example}[Navigation with Localized Uncertainty]
\label{Example:Navigation}
Figure~\ref{fig:toy_nav}(a) illustrates a 2D navigation task in which a point robot starts from a fixed location and must navigate to a goal region. The environment dynamics are accurately modeled across most of the state space, except within a designated \emph{uncertain} region (shown as the red shaded area), where the simulator becomes unreliable and entering may lead to poor prediction quality or potentially irrecoverable outcomes. In contrast, we seek a deployed policy that, before completing the primary task, safely approaches the uncertain region and moves along its boundary, the target manifold $\mathcal{U}$ (green region), to gather informative data while avoiding transitions into the unsafe interior.
\end{example}

\begin{figure}[t]
    \centering
    \includegraphics[width=0.41\linewidth]{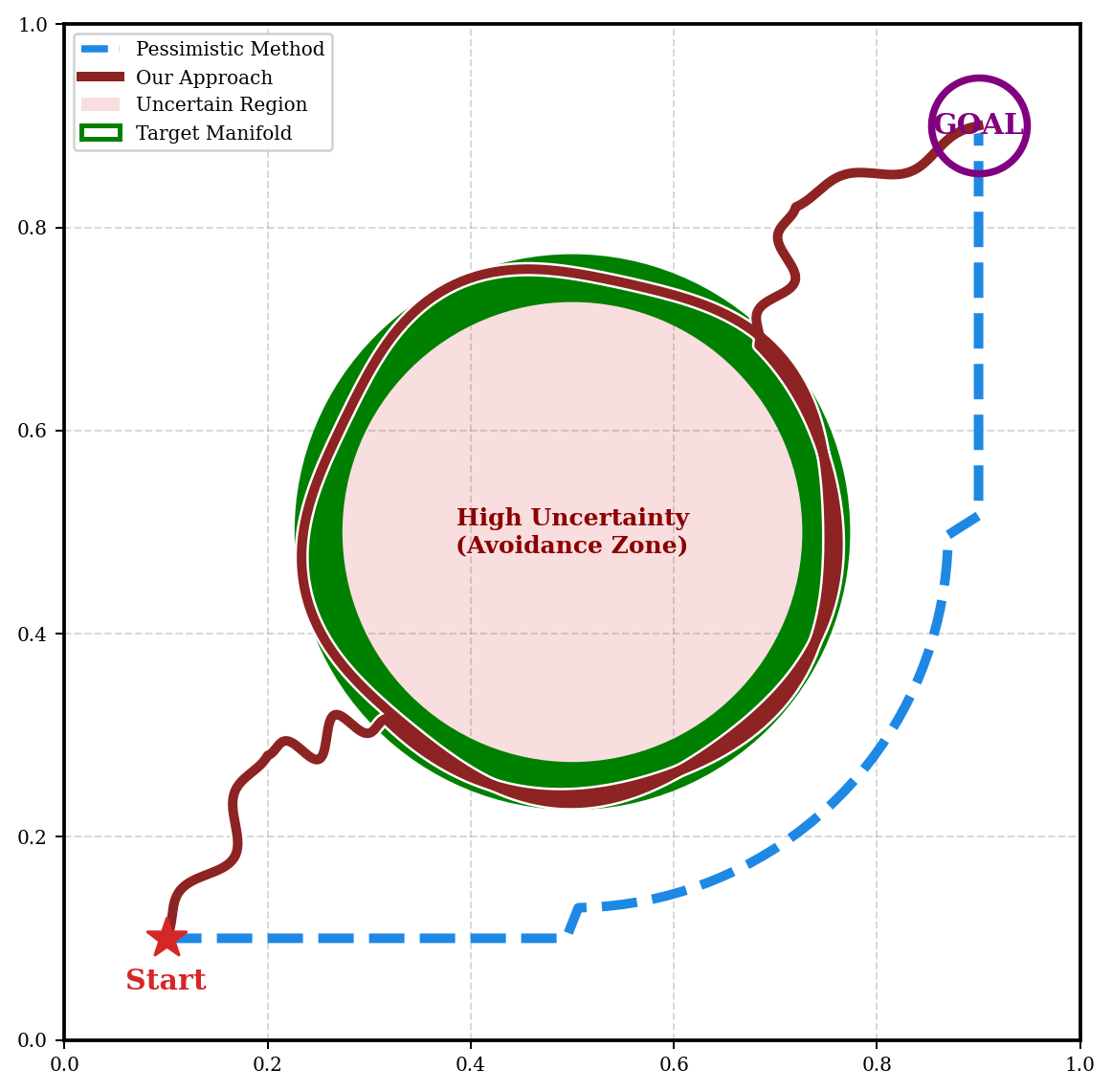}\hfill
    \includegraphics[width=0.41\linewidth]{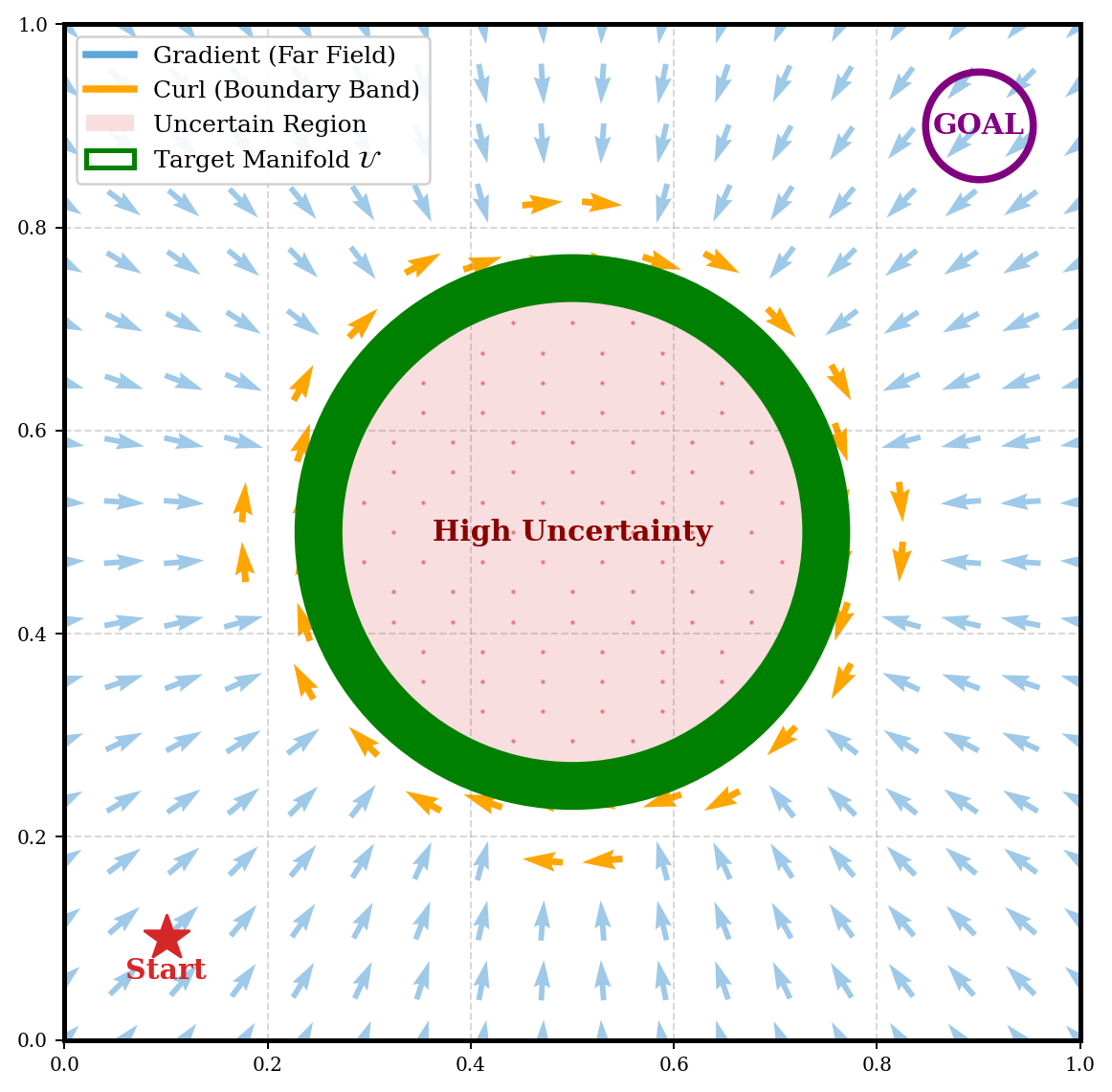}
    \caption{\textbf{Toy navigation with localized uncertainty.}
    \textbf{(a)} Example trajectories from the same start/goal. A pessimistic baseline (blue, dashed) takes a conservative detour to avoid the high-uncertainty region. Our approach drives the agent to the target manifold $\mathcal{U}$ (green), rotates along the boundary to collect informative samples, and then reaches the goal without entering the uncertain interior (shaded).
    \textbf{(b)} Visualization of the induced vector field for controlled exploration. The blue arrows  direct the agent toward the high-uncertainty regions, seeking the target uncertainty boundary $U_{\mathrm{mid}}$. Upon reaching this target level, the orange arrows (Curl Boundary Band) induce a tangential flow, ensuring the agent continuously moves along and explores the surface of the manifold.}
    \label{fig:toy_nav}
\end{figure}


\rev{To balance primary-task performance and safe boundary exploration}, we augment the base task reward $r(s,a)$ with a shaping term $r'(s,a,s')$, such that the total reward is $\tilde{r}(s,a,s') := r(s,a) + r'(s,a,s')$. For a state transition $s \xrightarrow{a} s'$, we define the agent's step vector as $\Delta_s := s' - s$. Our proposed shaping reward is formulated as follows:
\begin{equation}
\label{eq:rewardDesign}
\begin{aligned}
    r'(s,a,s') := \underbrace{\alpha(s)\langle \nabla U(s),\Delta_s\rangle}_{\text{Gradient Alignment}} + \underbrace{\beta(s)\langle W \nabla U(s), \Delta_s \rangle}_{\text{Rotational Flow}}
\end{aligned}
\end{equation}
At a high level, $r'$ generates a vector field over the continuous state space. This field is explicitly designed to attract the RL agent toward the uncertain region until it reaches the target safety level $U_{\mathrm{mid}}$, and subsequently encourages the agent to continuously move around the target manifold $\mathcal{U}$ (See Figure~\ref{fig:toy_nav}~(b)). We explain the mechanics of this equation step by step below.

\textbf{Gradient Alignment (Attraction):} The first element of our reward design acts as an attractive force to guide the policy toward the target manifold. The state-dependent scalar weight $\alpha(s) := \mathrm{sign}(U_{\mathrm{mid}}-U(s))\tanh(|U(s)-U_{\mathrm{mid}}|)$ dictates both the strength and direction of this alignment. When the agent is in a highly explored state where $U(s) < U_{\mathrm{mid}}$, $\alpha(s)$ is positive. This rewards the agent for taking a step $\Delta_s$ along the direction of the gradient $\nabla U(s)$, locally increasing its uncertainty. Conversely, if the agent breaches the boundary into highly unsafe, out-of-distribution regions where $U(s) > U_{\mathrm{mid}}$, $\alpha(s)$ becomes negative. This elegantly reverses the vector field, pushing the agent in the direction of $-\nabla U(s)$ to reduce uncertainty and safely guide it back to the recoverable boundary.

\textbf{Rotational Flow (Surface Exploration):} Once the agent reaches the target manifold, it is crucial that it does not simply stop exploring or "park" at a single boundary state. To prevent this, the second element of $r'$ induces a continuous, directed rotational flow. We isolate movement along the manifold by constructing a target vector field that is strictly tangent to $\mathcal{U}$. This is achieved by applying a constant skew-symmetric matrix $W$ (where $W^\top = -W$) to the gradient. Because the quadratic form of a skew-symmetric matrix is always zero ($\nabla U(s)^\top W \nabla U(s) = 0$), the resulting vector $W \nabla U(s)$ is mathematically guaranteed to be orthogonal to the gradient, thus pointing purely along the level set \citep{kapitanyuk2017guiding}. We measure the agent's step $\Delta_s$ against this rotational field using the inner product $\langle W \nabla U(s), \Delta_s \rangle$. The scalar $\beta(s) := 1-|\tanh(U(s)-U_{\mathrm{mid}})|$ acts as a gating mechanism that peaks when the agent is exactly on the manifold ($U(s) \approx U_{\mathrm{mid}}$). By rewarding this alignment, the agent is heavily incentivized to continuously "surf" along the manifold in a stable orbit, maximizing state coverage without penetrating deeper into unsafe regions. It is important to note that for state-space dimensions $d > 2$, the choice of the skew-symmetric matrix $W$ is not unique, as higher-dimensional spaces contain multiple orthogonal planes of rotation. 


One might naturally wonder if the gradient alignment term alone is sufficient to induce exploration once the manifold is reached. However, a purely gradient-based field cannot sustain continuous boundary traversal, a limitation we formalize through fundamental vector calculus:

\begin{lemma}[Insufficiency of Gradient Fields for Boundary Exploration]
\label{Lemma:necc_rot}
\textit{A reward function consisting solely of a gradient field yields zero net exploratory return for any trajectory confined to the target manifold $\mathcal{U}$, or for any closed loop.}
\end{lemma}
\textit{Proof.} Let $\gamma(t)$ for $t \in [0, T]$ be a continuous trajectory taken by the agent. By the fundamental theorem of line integrals \citep{spivak2018calculus}, the total reward accumulated from traversing a purely gradient field $\nabla U$ depends entirely on the endpoints: $\int_{\gamma} \nabla U(s) \cdot ds = U(\gamma(T)) - U(\gamma(0))$. If the agent perfectly navigates along the target manifold $\mathcal{U}$, the uncertainty remains constant, meaning $U(\gamma(T)) = U(\gamma(0)) = U_{\mathrm{mid}}$. Consequently, the net reward integral evaluates to exactly zero. $\hfill \blacksquare$

 Lemma~\ref{Lemma:necc_rot} necessitates our introduction of the rotational flow term, which breaks the conservative nature of the field and actively rewards the continuous traversal of the boundary.

\section{Analysis}
\label{sec:analysis}

In this section, we provide a theoretical analysis of the vector field reward shaping mechanism $\tilde{r}(s,a,s')$ introduced in Section~\ref{sec:approach}. Specifically, we demonstrate that in the long-term average, the gradient alignment term acts as a near-potential-based shaper \citep{ng1999policy} that cancels out, leaving the agent to optimize a combination of the base task reward and the rotational flow along the target manifold $\mathcal{U}$.   To guarantee these theoretical properties, we first establish a set of standard, mild assumptions about the environment dynamics, the structure of the uncertainty oracle $U$, and the feasibility of the target manifold.

\begin{assumption}[Regularity of the Target Level]
\label{assum:regularity}
The target value $U_{\mathrm{mid}}$ is a \emph{regular value} of $U$. Specifically, there exists $g_0 > 0$ such that $\|\nabla U(s)\|\ge g_0$ for all $s$ in a neighborhood of $\mathcal{U}$.
\end{assumption}
Assumption~\ref{assum:regularity} ensures our target set $\mathcal{U}$ is a well-behaved, smooth $(d-1)$-dimensional embedded manifold without ``pinched'' points or singularities \citep{lee2003smooth}. In particular, it guarantees that $\nabla U(s) \neq 0$ near the manifold, ensuring that the rotational field $W\nabla U(s)$ remains non-vanishing.

\begin{assumption}[Ergodicity and Average Reward]
\label{assum:ergodicity}
For every stationary policy $\pi$, the induced Markov chain is ergodic and admits a unique stationary distribution $\mu_\pi$. Consequently, the average reward $\rho^\pi := \lim_{T\to\infty}\frac{1}{T}\mathbb{E}_\pi\left[\sum_{t=0}^{T-1}\tilde r(s_t,a_t,s_{t+1})\right]$ always exists.
\end{assumption}
 Assumption~\ref{assum:ergodicity} is a standard infinite-horizon RL assumption. It allows us to evaluate policies based on their long-term stationary behavior, safely washing out the transient effects of initial state distributions \citep{levin2017markov}.

\begin{assumption}[Existence of a Manifold-Covering Policy]
\label{assum:manifold_policy}
There exist constants $v_{\max} < \infty$ and $v_0 > 0$, alongside a stationary policy $\pi_{\mathcal{U}}$, such that:
\begin{enumerate}
    \item[(i)] $\mu_{\pi_{\mathcal{U}}}(\mathcal{U}) = 1$,  i.e., the stationary
    state-visitation distribution of $\pi_{\mathcal{U}}$ is supported on
    $\mathcal{U}$.
    \item[(ii)] $\mathbb{E}_{\pi_{\mathcal{U}}}[\langle W \nabla U(s), \Delta_s \rangle] \ge v_0$, while for all policies $\pi$, \rev{$\mathbb{E}_{a\sim\pi(\cdot\mid s),\,s'}
\left[
\left|\langle W\nabla U(s),\Delta_s\rangle\right|
\,\middle|\,s
\right]
\le v_{\max}$}. 
\end{enumerate}
\end{assumption}

Assumption~\ref{assum:manifold_policy} should be viewed as a controllability-type condition. It posits that the environment admits at least one stationary policy that can stay on the manifold $\mathcal{U}$ while maintaining nontrivial tangential motion in the direction induced by $W\nabla U(s)$. This is natural in environments where the agent can steer locally along the level set while correcting drift in the normal direction. Our navigation experiment in Section~\ref{sec:experiments} provides a concrete instance of this setting.


\subsection{Theoretical Results}
\label{sec:no_sticking}
With the formal assumptions established, we now provide rigorous theoretical guarantees for our proposed vector-field reward. The first fundamental requirement of our approach is to ensure that the agent reliably navigates to, and remains tightly concentrated around, the target boundary $\mathcal{U}$. Theorem \ref{thm:manifold_shaping_revised} formalizes this manifold-seeking behavior by demonstrating that the gradient-alignment component of our reward naturally acts as a potential-based shaping mechanism.
\begin{theorem}[Manifold-seeking orthogonal reward shaping]
\label{thm:manifold_shaping_revised}
Let Assumptions \ref{assum:regularity}, \ref{assum:ergodicity}, and \ref{assum:manifold_policy} hold.  Let $\Phi(s) := |U(s) - U_{\mathrm{mid}}|$ and define the scalar potential function  $\Psi(s) := \log\cosh(\Phi(s))$. Let 
$L_\Psi := \sup_{x\in\mathcal{S}}\|\nabla^2\Psi(x)\|_{\mathrm{op}} < \infty$ denote its maximum curvature.  Then the following statements hold:

\textbf{1. Potential-based decomposition up to second-order error.} For all transitions $(s,a,s')$, the normal-seeking component acts as a potential difference:
\begin{equation}
\label{eq:Potential_SecondOrder_error}
    \begin{aligned}
& \alpha(s)\langle \nabla U(s),\Delta_s\rangle = -\big(\Psi(s')-\Psi(s)\big) + \varepsilon(s,a,s')
 \end{aligned}
\end{equation}
where the Taylor remainder is bounded by $|\varepsilon(s,a,s')| \le \frac{L_\Psi}{2}\|\Delta_s\|_2^2$.

\textbf{2. Vanishing potential in stationary average.} For any stationary policy $\pi$, the long-term average shaped reward simplifies to:
\begin{equation}
\label{eq: Vanishing_potnetial}
\rho^\pi = \bar\rho_r^\pi + \mathbb{E}_{s\sim \mu_\pi, a\sim \pi}\Big[\mathbb{E}_{s'}\big[\beta(s)\langle W \nabla U(s), \Delta_s \rangle\big]\Big] + \bar\varepsilon^\pi
\end{equation}
where $\bar\rho_r^\pi$ is the average base reward, and the error satisfies $|\bar\varepsilon^\pi| \le \frac{L_\Psi}{2}\mathbb{E}_{\mu_\pi}[\|\Delta_s\|_2^2]$. \rev{Thus, in the small-step regime, the stationary objective consists of the base reward and tangential reward up to a second-order error.}

\textbf{3. Near-manifold concentration.} Fix $\epsilon > 0$ and define the $\epsilon$-off-manifold region $N_\epsilon := \{s : \Phi(s)\ge \epsilon\}$. Let $b_\epsilon := 1-\tanh(\epsilon)$. Any optimal policy $\pi^\star \in \arg\max_\pi \rho^\pi$ satisfies:
\begin{equation}
\label{eq:concentration_Rlc_1}
\mu_{\pi^\star}(N_\epsilon) \le \frac{\big(\bar\rho_r^{\pi^\star}-\bar\rho_r^{\pi_{\mathcal{U}}}\big) + (v_{\max}-v_0) + |\bar\varepsilon^{\pi^\star}| + |\bar\varepsilon^{\pi_{\mathcal{U}}}|}{(1-b_\epsilon)\,v_{\max}}.
\end{equation}

\end{theorem}
\begin{remark}[Behavioral Trade-offs in Near-Manifold Concentration]

Equation~\ref{eq:concentration_Rlc_1} makes the behavioral trade-offs of our reward design fully explicit. 
If we ignore the second-order discretization errors ($|\bar{\varepsilon}^{\pi^*}|$ and $|\bar{\varepsilon}^{\pi_\mathcal{U}}|$), 
the upper bound on the off-manifold state visitation $\mu_{\pi^*}(\mathcal{N}_\epsilon)$ is driven primarily by two structural terms:

\begin{itemize}

\item \textbf{Base Reward Temptation:} 
The first term, $\frac{\bar{\rho}_r^{\pi^*} - \bar{\rho}_r^{\pi_\mathcal{U}}}{v_{\max}}$, captures the maximum additional task reward the optimal policy might gain by venturing off the manifold. 
Since $v_{\max} \ge v_0$, this fraction is strictly bounded above by 
$\frac{\bar{\rho}_r^{\pi^*} - \bar{\rho}_r^{\pi_\mathcal{U}}}{v_0}$. 
Recall that $v_0$ lower-bounds the expected rotational reward of the strictly on-manifold policy,\(
\mathbb{E}_{\pi_{\mathcal{U}}}\!\left[ \beta(s)\langle W \nabla U(s), \Delta_s \rangle\right] \ge v_0.
\) Therefore, if the policy can move rapidly along the boundary (i.e., $v_0$ is large), this term becomes small.  Furthermore, if the base task reward difference is naturally small, the ``temptation'' to leave the manifold is negligible, resulting in a tighter concentration bound.

\item \textbf{Tangential Velocity Gap:} 
The second fraction, $\frac{v_{\max} - v_0}{v_{\max}}$, reflects the relative gap in exploration speed.  If the on-manifold policy ($\pi_{\mathcal{U}}$) can traverse the boundary rapidly, its guaranteed tangential speed $v_0$ approaches the global maximum possible speed $v_{\max}$, effectively driving this difference to zero.

\end{itemize}

\rev{Thus, concentration improves when the task-reward advantage of leaving the manifold is small and the on-manifold policy achieves near-maximal tangential
motion. Moreover, scaling $W$ by a constant $c>0$ scales both $v_0$ and $v_{\max}$ while leaving the base reward unchanged, explicitly controlling the relative strength of exploration and primary-task performance.}


\end{remark}


\begin{proof}
The proof of Eq.~\eqref{eq:Potential_SecondOrder_error} is deferred to Appendix~\ref{app:proof_claim1}. We now proceed to prove Eqs.~\eqref{eq: Vanishing_potnetial} and~\eqref{eq:concentration_Rlc_1}.

\textbf{Step 1: Average-reward limit (Proof of Eq.~\eqref{eq: Vanishing_potnetial}).}
Fix a stationary policy $\pi$. Summing $-(\Psi(s_{t+1})-\Psi(s_t))$ over $t=0,\dots,T-1$ telescopes to $\Psi(s_0)-\Psi(s_T)$. Since $\Psi$ is bounded on the compact set $\mathcal{S}$, taking the expectation and dividing by $T$ yields $\lim_{T\to\infty}\frac{1}{T}\mathbb{E}[\Psi(s_0)-\Psi(s_T)]=0$. 
Thus, under Assumption \ref{assum:ergodicity}, the potential difference vanishes in the limit. Substituting the result of Eq.\eqref{eq:Potential_SecondOrder_error} into the definition of $\tilde{r}$ leaves only the base reward, the tangential term, and the stationary expected Taylor remainder $\bar\varepsilon^\pi$, proving Claim 2.

\textbf{Step 2: Near-manifold concentration (Proof of Eq.~\eqref{eq:concentration_Rlc_1}).}
Fix $\epsilon>0$. On the off-manifold region $N_\epsilon$, $w(s) = \tanh(\Phi(s)) \ge \tanh(\epsilon)$, which implies $\beta(s) = 1-w(s) \le 1-\tanh(\epsilon) = b_\epsilon$. Because $\beta(s)\le 1$ everywhere and conditional expected tangential movement is bounded by $v_{\max}$ (Assumption \ref{assum:manifold_policy}), we can bound the tangential reward for any policy $\pi$:
\begin{align*}
\mathbb{E}\big[\beta(s)\langle W \nabla U(s), \Delta_s \rangle\big] &\le v_{\max}\Big( (1-\mu_\pi(N_\epsilon))\cdot 1 + \mu_\pi(N_\epsilon)\cdot b_\epsilon \Big) \\
&= v_{\max}\Big(1-(1-b_\epsilon)\mu_\pi(N_\epsilon)\Big).
\end{align*}
Applying this to the optimal policy $\pi^\star$ via Claim 2:
\[
\rho^{\pi^\star} \le \bar\rho_r^{\pi^\star} + v_{\max}\Big(1-(1-b_\epsilon)\mu_{\pi^\star}(N_\epsilon)\Big) + |\bar\varepsilon^{\pi^\star}|.
\]
Conversely, for the manifold-invariant policy $\pi_{\mathcal{U}}$ from Assumption \ref{assum:manifold_policy}, $\Phi(s)=0$ almost surely. Thus $\beta(s)=1$ a.s., yielding:
\[
\rho^{\pi_{\mathcal{U}}} \ge \bar\rho_r^{\pi_{\mathcal{U}}} + v_0 - |\bar\varepsilon^{\pi_{\mathcal{U}}}|.
\]
By optimality, $\rho^{\pi^\star} \ge \rho^{\pi_{\mathcal{U}}}$. Chaining these inequalities and rearranging for $\mu_{\pi^\star}(N_\epsilon)$ directly yields the bound in Eq.~\eqref{eq:concentration_Rlc_1}.
\end{proof}

\textbf{Tangential Mobility and No-Sticking Guarantees.} While Theorem~\ref{thm:manifold_shaping_revised} guarantees that the agent reaches and remains concentrated around the target manifold $\mathcal{U}$, this alone is insufficient for effective data collection. A common failure mode in boundary-seeking reinforcement learning is \emph{degenerate parking}: the agent reaches a safe boundary state and stops moving—often via a self-loop policy—to avoid the risk of out-of-distribution penalties. To rule out this behavior, we show that the rotational component of our shaped reward prevents such degeneracy. The following corollary formalizes this result: 

\begin{corollary}[No-sticking on the manifold (in expectation)]
\label{cor:no_sticking}
Let the setting of Theorem~\ref{thm:manifold_shaping_revised} hold and suppose Assumption~\ref{assum:manifold_policy} holds with constants $v_0>0$ and $v_{\max}<\infty$.
Further assume that the base reward of the manifold-supported policy $\pi_{\mathcal{U}}$ is near-optimal among all policies supported on $\mathcal{U}$, in the sense that for some tolerance $\eta\ge 0$, \(
\bar\rho_r^{\pi_{\mathcal{U}}}
\;\ge\;
\sup_{\pi:\mu_\pi(\mathcal{U})=1}\bar\rho_r^\pi
\;-\;\eta.
\) Then any average-reward optimal policy $\pi^\star\in\arg\max_\pi \rho^\pi$ that satisfies $\mu_{\pi^\star}(\mathcal{U})=1$ must achieve strictly positive expected tangential motion:
\[
\mathbb{E}_{s\sim\mu_{\pi^\star}}\mathbb{E}_{a,s'}\!\left[\beta(s)\langle W \nabla U(s), \Delta_s \rangle\right]
\;\ge\;
v_0-\eta.
\]
In particular, if $\eta < v_0$, then $\pi^\star$ cannot induce a self-loop ($s'=s$ almost surely) on a set of stationary mass $1$.
\end{corollary}

\begin{proof}
On the target manifold $\mathcal{U}$, the shaping weights satisfy $\alpha(s)=0$ and $\beta(s)=1$, hence
\[
\tilde r(s,a,s') = r(s,a) + \beta(s)\langle W \nabla U(s), \Delta_s \rangle.
\]
By Assumption~\ref{assum:manifold_policy}(i)--(ii), the invariant policy $\pi_{\mathcal{U}}$ satisfies $\mu_{\pi_{\mathcal{U}}}(\mathcal{U})=1$ and
\[
\rho^{\pi_{\mathcal{U}}}
=\bar\rho_r^{\pi_{\mathcal{U}}}+
\mathbb{E}_{\pi_{\mathcal{U}}}\!\left[\beta(s)\langle W \nabla U(s), \Delta_s \rangle\right]
\;\ge\; \bar\rho_r^{\pi_{\mathcal{U}}}+ v_0.
\]
Let $\pi^\star$ be average-reward optimal among policies supported on $\mathcal{U}$ (i.e., $\mu_{\pi^\star}(\mathcal{U})=1$). Optimality implies
$\rho^{\pi^\star}\ge \rho^{\pi_{\mathcal{U}}}$, and expanding $\rho^{\pi^\star}$ on $\mathcal{U}$ yields
\[
\bar\rho_r^{\pi^\star}+
\mathbb{E}_{\pi^\star}\!\left[\beta(s)\langle W \nabla U(s), \Delta_s \rangle\right]
\;\ge\;
\bar\rho_r^{\pi_{\mathcal{U}}}+ v_0.
\]
Rearranging gives
\[
\mathbb{E}_{\pi^\star}\!\left[\beta(s)\langle W \nabla U(s), \Delta_s \rangle\right]
\;\ge\;
v_0 - (\bar\rho_r^{\pi_{\mathcal{U}}}-\bar\rho_r^{\pi^\star}).
\]
Finally, the stated base-reward condition implies $\bar\rho_r^{\pi^\star}\le \bar\rho_r^{\pi_{\mathcal{U}}}+\eta$, hence
\[
\mathbb{E}_{\pi^\star}\!\left[\beta(s)\langle W \nabla U(s), \Delta_s \rangle\right]
\;\ge\;
v_0-\eta,
\]
as claimed. If $\eta<v_0$, the right-hand side is strictly positive, ruling out an almost-sure self-loop on a full-mass stationary set.
\end{proof}

\section{Experiments}
\label{sec:experiments}

We evaluate our method on the continuous navigation task introduced in Example~\ref{Example:Navigation} and illustrated in Figure~\ref{fig:toy_nav}. Our primary objective is to examine whether an offline-trained fixed policy can perform its primary task while actively
collecting informative samples from the mid-level uncertainty boundary $\mathcal{U}$. To separate the behavior induced by the exploratory reward from
its interaction with the task reward, we consider two settings:

\begin{itemize}\item \textbf{Pure Exploration (Extrinsic Reward Removed):} In the first scenario, we completely zero out the task-specific environment reward. This isolates the effects of our shaping reward, allowing us to observe the agent's pure exploratory dynamics and verify whether it successfully adheres to and circulates along the manifold $\mathcal{U}$ without the pull of an external goal.\item \textbf{Goal-Directed Exploration (Extrinsic Reward Active):} In the second scenario, we reintroduce the standard environment reward, which incentivizes the agent to navigate toward a designated goal region. This tests the policy's ability to balance the exploratory behavior induced by our vector-field reward with standard task performance.
\end{itemize}

Throughout these navigation scenarios, we benchmark our proposed method against an intrinsic reward that assigns higher value to states with greater estimated uncertainty. We compare the policies learned under these two reward designs to evaluate their resulting exploration behavior. Finally, we conclude by discussing the limitations of our current methodology and outlining interesting open problems for future research.

\textbf{Implementation Details.}
We train Soft Actor-Critic (SAC)
\citep{haarnoja2018soft, haarnoja2018soft_2} with a normalizing-flow actor \citep{ghugare2025normalizing}. We use 32 flow blocks for both settings, with an actor learning rate of $1\times10^{-6}$. \rev{All experiments are run with five independent random seeds.} The environment and training pipeline are implemented in JAX and run on an NVIDIA H200 GPU. Full environment and optimization details are provided in Appendix~\ref{app:Hyperparams}.

\subsection{Analysis of Exploratory Behavior}\label{subsec:exploration_results}
We first study the exploratory behavior induced by the intrinsic reward. To isolate this effect, we remove the extrinsic reward for reaching the goal region, so the agent's objective is driven entirely by the exploratory signal. We then compare two approaches under this purely exploratory setting:

\textbf{(i) Vector-field exploratory reward (Ours).} We train the agent using the reward in Eq.~\eqref{eq:rewardDesign}. Because the
state space is two-dimensional, we use the standard skew-symmetric matrix
\( W=
\begin{bmatrix}
0 & 1\\
-1 & 0
\end{bmatrix}. \) Figure~\ref{fig:four_experiments} reports three diagnostics of the learned behavior.
Panel (a) visualizes empirical state visitations during an episode, illustrating the spatial coverage of the policy.
Panel (b) measures the agent's speed along the target manifold, while panel (c) reports the rate of unsafe transitions. During training, the agent gradually transitions from radial motion toward tangential motion around the target uncertainty level. 
This behavior is visible in the visitation map (panel (a)), where the states form a continuous band around the uncertainty boundary, and in panel (b), where the speed along the manifold increases as training progresses. Consistent with Corollary~\ref{cor:no_sticking}, the policy does not stall after reaching the manifold. 
Instead, the rotational component of the vector field induces a persistent orbit around the boundary, producing well-distributed coverage of the exploration frontier while maintaining a low unsafe rate (panel (c)).


\begin{figure*}[t] 
    \centering
    \begin{subfigure}[b]{0.3\textwidth}
        \centering
        \includegraphics[width=\textwidth]{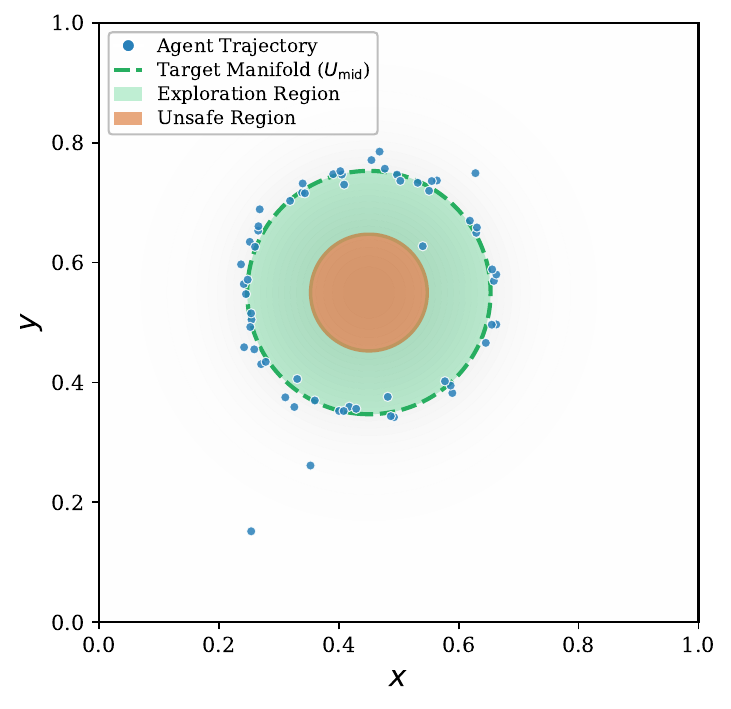} 
        \caption{Exploration behavior}
        \label{fig:exp_a_vf_method_particles}
    \end{subfigure}
    \hfill
    \begin{subfigure}[b]{0.3\textwidth}
        \centering
        \includegraphics[width=\textwidth]{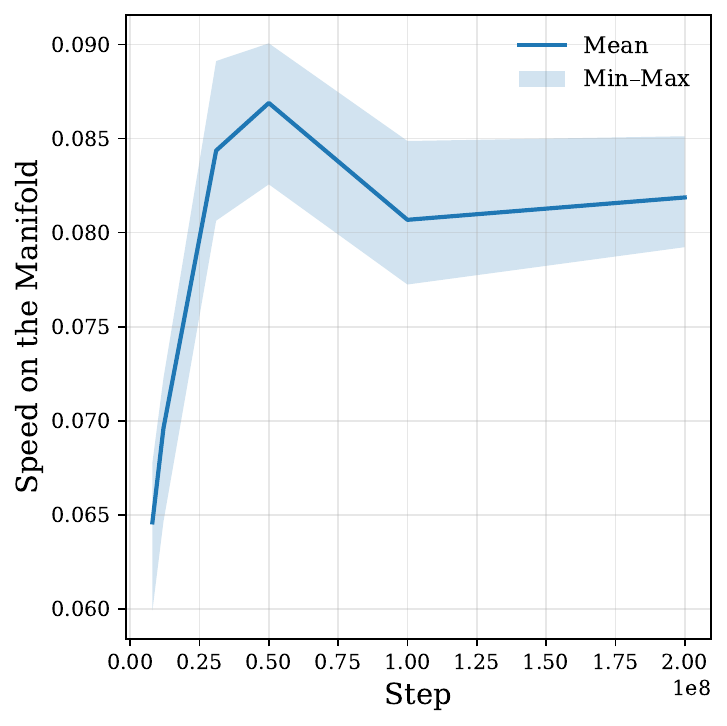}
        \caption{Speed on the Manifold}
        \label{fig:exp_b}
    \end{subfigure}
    \hfill
    \begin{subfigure}[b]{0.3\textwidth}
        \centering
        \includegraphics[width=\textwidth]{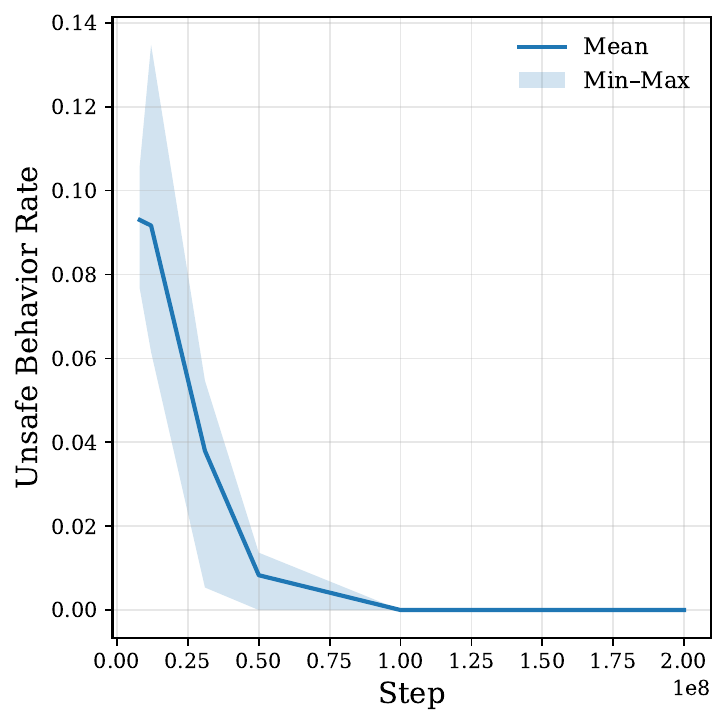}
        \caption{Unsafe Behavior Rate}
        \label{fig:exp_d}
    \end{subfigure}
    \caption{Our vector field reward design results in periodic orbit behavior, where the agent rotates around the uncertainty region to collect as much data as possible.}
    \label{fig:four_experiments}
\end{figure*}

\textbf{(ii) Purely State-Based Intrinsic Reward Baseline.} 
\rev{To highlight the necessity of the rotational flow, we construct a baseline that
directly incentivizes reaching the target uncertainty boundary using a purely
state-based bump reward. Let \(\phi(s)=\frac{U(s)}{U_{\mathrm{mid}}}\). The baseline uses
\(r'_{\mathrm{bump}}(s,a,s')
\propto
\exp\!\left(-\kappa(\phi(s')-1)^2\right),
\) so that states near \(U(s')=U_{\mathrm{mid}}\) receive the largest intrinsic
reward. This reward can be combined with a penalty for states that move too far
beyond the target boundary.}



One can view this reward structure as encouraging the agent to match a target distribution $p^*(s)$ that is heavily concentrated on the target manifold $\mathcal{U}$, conceptually similar to the objective in the State Marginal Matching (SMM) framework \citep{lee2019efficient}, except that we remove the term involving the policy's state visitation distribution. \rev{ As shown in Figure~\ref{fig:four_experiments_bump_reward}, the bump reward attracts the agent to the target manifold, but the visitation remains confined to a localized portion of the boundary rather than spreading across the entire frontier. The visible spread is primarily due to entropy-driven stochasticity: the agent exhibits localized random fluctuations near the boundary, but the resulting clockwise and counterclockwise motions cancel in expectation, as reflected by the near-zero speed on the manifold in Figure~\ref{fig:four_experiments_bump_reward}(b). Thus, the bump reward does not produce coherent movement along the target manifold. This phenomenon is consistent with the observations of \citet{lee2019efficient}, who note that optimizing the unnormalized log-likelihood of a target distribution can lead to mode collapse. Their framework addresses this issue by incorporating a term involving the policy's state visitation distribution, which we discuss in Remark~\ref{remark:smm_challenges}.}


\begin{figure*}[t] 
    \centering
    \begin{subfigure}[b]{0.3\textwidth}
        \centering
        \includegraphics[width=\textwidth]{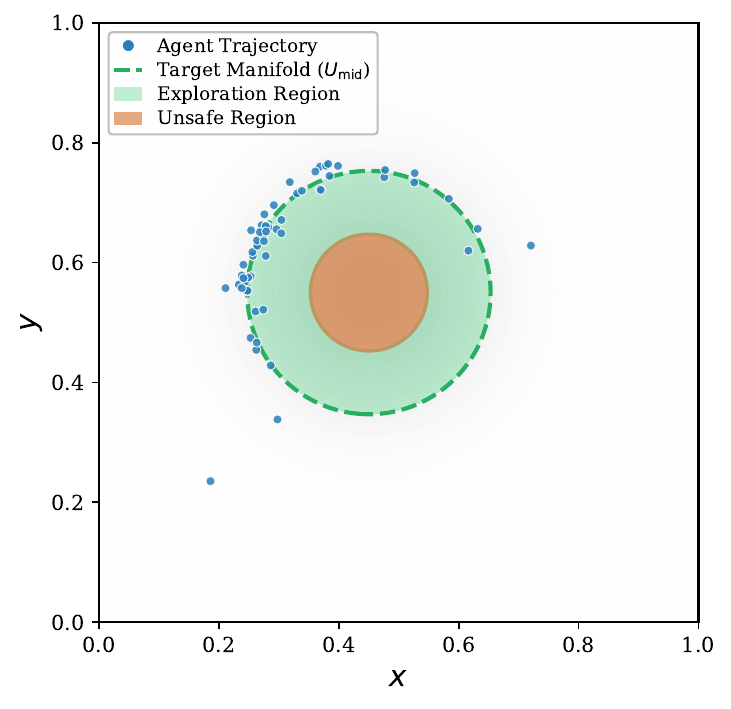} 
        \caption{Exploration behavior}
        \label{fig:exp_a_bump_method_particles}
    \end{subfigure}
    \hfill
    \begin{subfigure}[b]{0.3\textwidth}
        \centering
        \includegraphics[width=\textwidth]{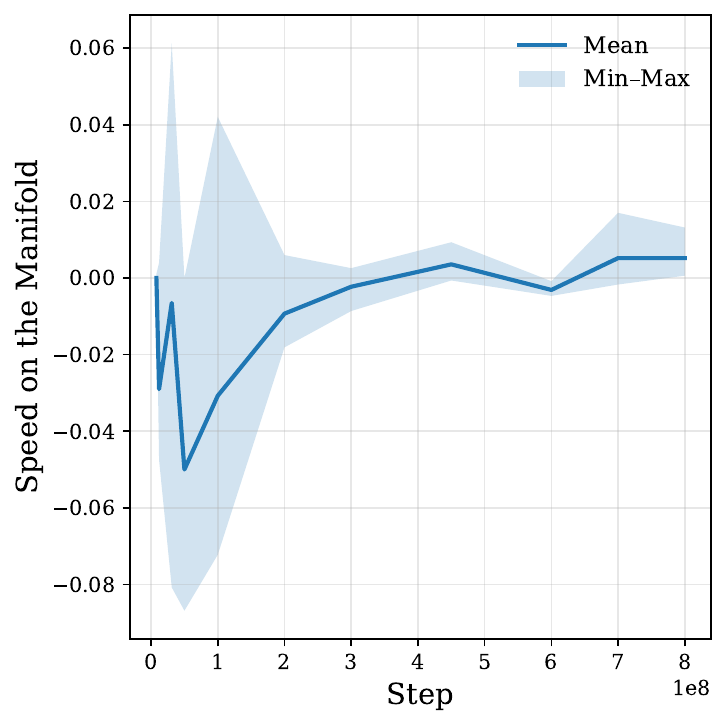}
        \caption{Speed on the Manifold}
        \label{fig:exp_b_bump}
    \end{subfigure}
    \hfill
    \begin{subfigure}[b]{0.3\textwidth}
        \centering
        \includegraphics[width=\textwidth]{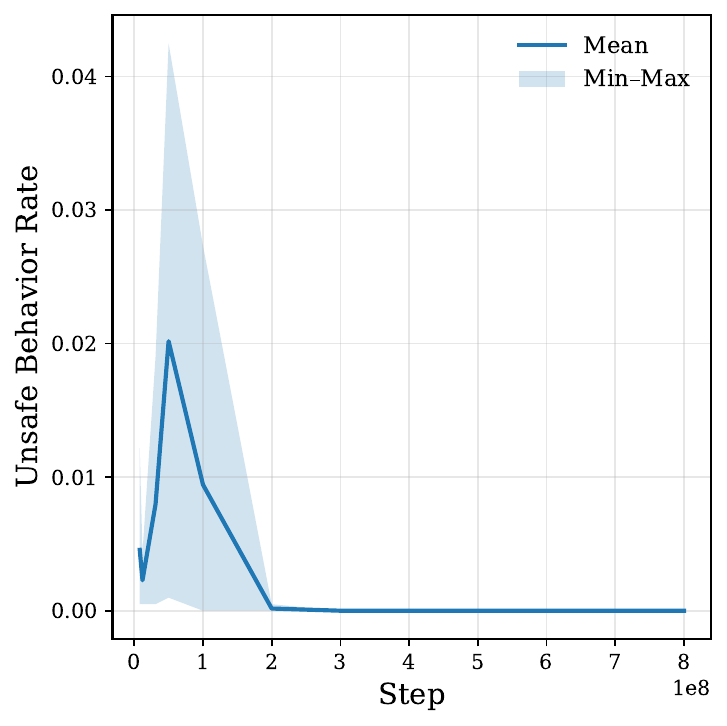}
        \caption{Unsafe Behavior Rate}
        \label{fig:exp_d_bump}
    \end{subfigure}
    \caption{ State-based intrinsic reward baseline demonstrating mode collapse. Without the rotational flow component, the agent's state visitations (blue points) collapse into a highly localized cluster on the target manifold $\mathcal{U}$ rather than continuously exploring the frontier.}
    \label{fig:four_experiments_bump_reward}
\end{figure*}

\begin{remark}[Challenges of State Marginal Matching for Stationary Deployment]\label{remark:smm_challenges}
To mitigate the parking behavior described above, the complete SMM framework introduces a correction term that subtracts the policy's own state visitation distribution from the reward, yielding an objective of the form $r(s) \propto \log p^*(s) - \log(\text{policy state visitation})$ \citep{lee2019efficient}. As discussed by \citet{lee2019efficient}, this quantity must be estimated from the data generated by the current policy and updated whenever the policy changes during training. Consequently, accurately maintaining this estimate requires continuous density estimation throughout online interaction. This requirement moves the problem outside the scope of standard reinforcement learning, since the reward function itself depends on the policy. Maintaining and updating such an estimator during an online deployment phase is computationally demanding and can introduce additional instability. Furthermore, achieving full coverage of the target distribution $p^*(s)$ in the SMM framework often requires training and deploying a mixture of policies (or ``experts''), which can be memory-intensive and difficult to scale to continuous control settings. In contrast, our vector-field reward introduces an explicit rotational component that encourages tangential motion along the manifold, enabling diverse boundary exploration with a single stationary neural network policy and without requiring online estimation of the policy's visitation distribution. That said, we emphasize that our reward shaping is not a universal solution and may have its own limitations. We refer the reader to Section~\ref{sect:limit} for a detailed discussion of potential limitations and open problems.
\end{remark}

\subsection{Balancing Exploration with Task Completion}
\rev{We next restore the environment reward for reaching the goal and evaluate the exploration--task trade-off within a fixed episode horizon. Because reaching the goal terminates the episode, the policy must decide how much of the pre-termination horizon to devote to boundary exploration.} We therefore augment the state representation with a time embedding. Figure~\ref{fig:four_experiments_mainTask} compares the performance of policies trained using our vector-field reward design against those trained with the state intrinsic reward baseline. As shown in panel (a), our method exhibits a distinct time-splitting strategy: the agent initially collects diverse samples along the target manifold, and then, in the final phase of the horizon, runs toward the goal to secure the primary task reward. In contrast, the agent trained with the state intrinsic reward fails to achieve comprehensive coverage, merely visiting one of the uncertain regions on its direct path to the goal. \rev{Additional five-seed training logs are provided in
Appendix~\ref{app:goal_directed_diagnostics}.}

\begin{figure*}[t] 
    \centering
    \begin{subfigure}[b]{0.4\textwidth}
        \centering
        \includegraphics[width=\textwidth]{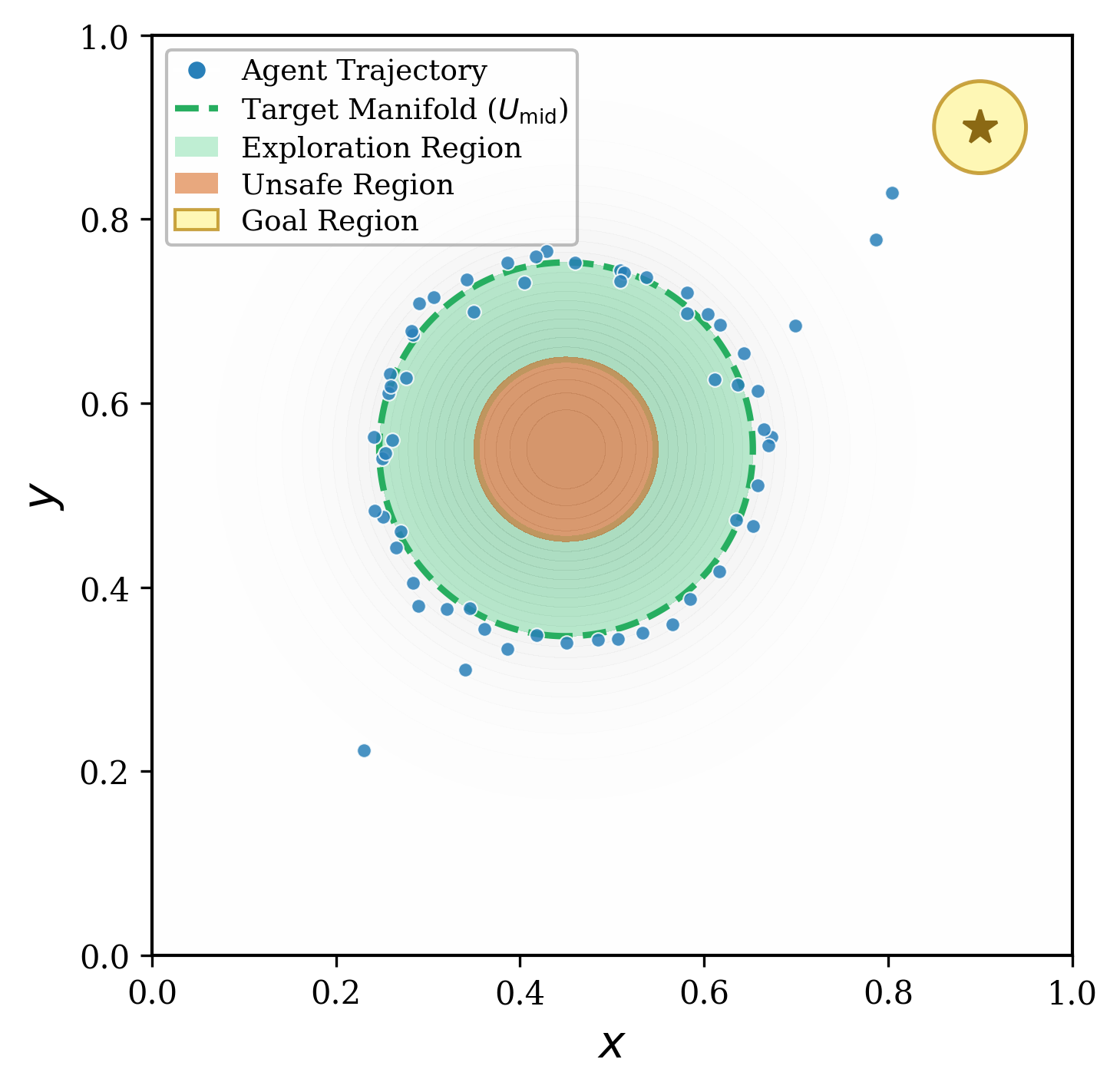} 
        \caption{Vector-field reward (Ours)}
        \label{fig:exp_a_vfMainTask}
    \end{subfigure}
    \hfill
    \begin{subfigure}[b]{0.4\textwidth}
        \centering
        \includegraphics[width=\textwidth]{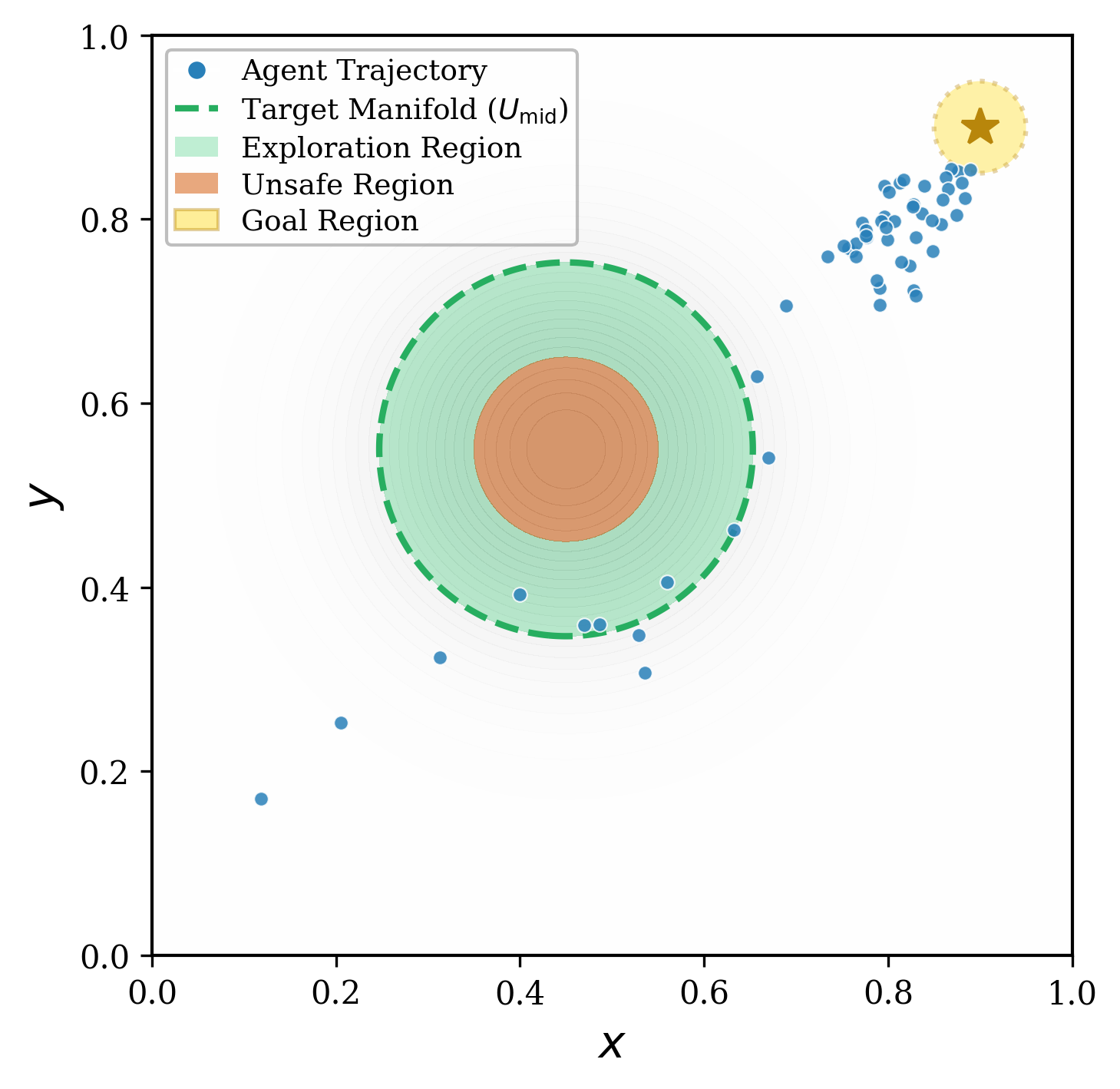}
        \caption{State intrinsic reward (Baseline)}
        \label{fig:exp_bump_mainTask}
    \end{subfigure}
    \caption{Balancing target manifold coverage with a primary navigation task. (a) Our method induces a time-splitting strategy, actively circulating the manifold before navigating to the goal. (b) The baseline intrinsic reward fails to achieve diverse boundary coverage, moving almost directly to the goal.}
    \label{fig:four_experiments_mainTask}
\end{figure*}

\subsection{Limitations and Future Work}
\label{sect:limit}
As established by our theoretical results (e.g., Corollary~\ref{cor:no_sticking}), our vector-field reward design successfully motivates the agent to continuously navigate along the target manifold $\mathcal{U}$. This highlights a crucial insight for boundary exploration: simply assigning high reward values to states residing on the manifold is insufficient, as it often leads to degenerate "parking" or stalling behaviors. Instead, active tangential movement is required to collect diverse, informative samples. However, the geometric nature of our approach introduces new challenges when scaling to higher-dimensional continuous control tasks. In a 2D state space, the target manifold is a 1D curve, meaning the rotational flow is essentially unique up to its direction (clockwise or counter-clockwise). Conversely, for $d > 2$, higher-dimensional manifolds lack a single canonical rotation \citep{spivak2018calculus}. Consequently, the skew-symmetric matrix $W$ in Eq.~\eqref{eq:rewardDesign} is highly underdetermined and can be chosen arbitrarily. As illustrated by the schematic in Figure~\ref{fig:vf_w}, different choices of $W$ induce completely distinct directional flows. A primary limitation of using a single, static $W$ matrix in higher dimensions is the risk of "orbit collapse." Even with our reward design, the agent might converge to rotating along a single, localized 1D orbit on the manifold rather than achieving comprehensive coverage of the entire $(d-1)$-dimensional surface. Addressing this topological limitation presents exciting avenues for future research. One potential approach is to combine a strong rotational flow with maximum entropy methods \citep{haarnoja2018soft}. Depending on the underlying geometry of the manifold, this noise injection could force the agent to traverse the boundary in a "zig-zag" pattern, naturally widening the exploration band. A more systematic solution could draw inspiration from the goal-conditioned RL literature\citep{nair2018visual, eysenbach2018diversity, ghugare2025normalizing}. By treating the matrix $W$ as a conditional input to the policy $\pi(a \mid s, W)$, we can train an agent capable of executing various rotational flows. During deployment, randomly sampling or systematically rotating through different $W$ matrices would generate intersecting search trajectories, maximizing surface coverage. Ultimately, while our proposed reward design does not universally solve the full distribution coverage problem addressed by State Marginal Matching \citep{lee2019efficient}, it offers a fundamentally different and tractable perspective. By leveraging $W$-conditioned geometric flows, future work could potentially bypass SMM's computationally heavy requirements for continuous online density estimation and game-theoretic policy mixtures, framing manifold exploration purely as a controllable dynamical system.

\begin{figure}[t]
    \centering
    \includegraphics[width=0.37\linewidth]{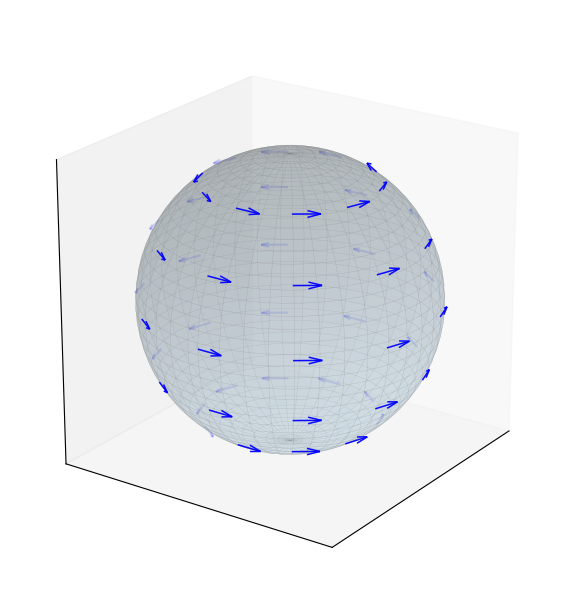}\hfill
    \includegraphics[width=0.37\linewidth]{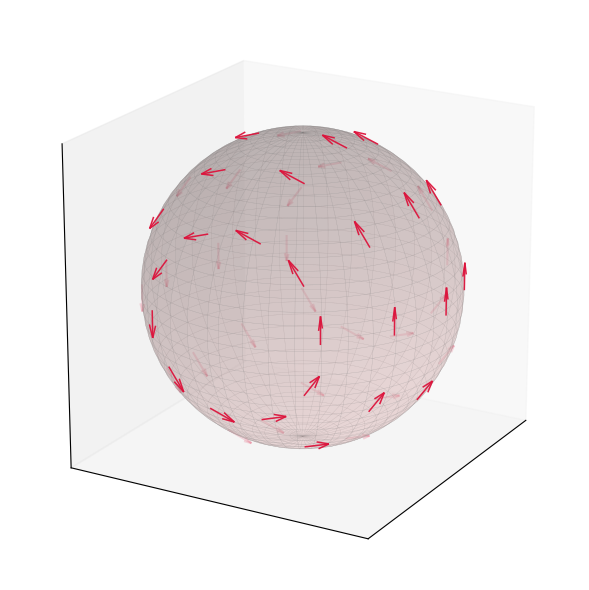}
    \caption{For higher dimensions ($d > 2$), our reward function in Eq.~\eqref{eq:rewardDesign} generates different tangential vector fields for different choices of the skew-symmetric matrix $W$. Because rotational flow is not unique in these spaces, each arbitrary $W$ dictates a distinct rotational trajectory along the boundary $\mathcal{U}$.}
    \label{fig:vf_w}
\end{figure}

\section{Conclusion}
In this work, we addressed the critical challenge of safely gathering informative out-of-distribution data during the deployment of pre-trained RL agents. By introducing a novel vector-field exploratory reward, we demonstrated that coupling gradient alignment with a rotational flow effectively compels a stationary policy to continuously navigate along target uncertainty boundaries. This geometric approach mitigates the degenerate ``parking'' behaviors associated with purely state-based intrinsic rewards and offers an alternative to relying on risky online policy updates. Ultimately, our framework provides a scalable, reliable mechanism for autonomous systems to balance primary task execution with safe frontier exploration, paving the way for robust iterative fine-tuning and safer sim-to-real transfer in continuous control domains.

\subsubsection*{Acknowledgments}
\label{sec:ack}
\rev{This work has been supported in part by the Army Research Laboratory and was accomplished under Cooperative Agreement Numbers W911NF-24-2-0205 and W911NF-23-2-0225, by the U.S. National Science Foundation under the grants: NSF AI Institute (AI-EDGE) 2112471, CNS2312836, CNS-2225561, and CNS-2239677, and by the Office of Naval Research under grant N00014-24-1-2729. The views and conclusions contained in this document are those of the authors and should not be interpreted as representing the official policies, either expressed or implied, of the Army Research Laboratory or the U.S. Government. The U.S. Government is authorized to reproduce and distribute reprints for Government purposes
notwithstanding any copyright notation herein.}
\bibliography{main}
\bibliographystyle{rlj}

\beginSupplementaryMaterials
\section{Proof of Claim 1 from Theorem~\ref{thm:manifold_shaping_revised}}
\label{app:proof_claim1}
\textbf{Step 1: Rewrite the normal term via the smooth potential $\Psi$.}
For $U(s)\neq U_{\mathrm{mid}}$, we have $\nabla \Phi(s)=\mathrm{sign}(U(s)-U_{\mathrm{mid}})\nabla U(s)$.
Using the definition $\alpha(s) = -w(s)\mathrm{sign}(U(s)-U_{\mathrm{mid}})$, we find:
\[
\alpha(s)\nabla U(s) = -w(s)\,\nabla \Phi(s).
\]
By the chain rule, the gradient of our defined potential $\Psi(s) = \log\cosh(\Phi(s))$ is:
\[
\nabla \Psi(s) = \tanh(\Phi(s))\,\nabla \Phi(s) = w(s)\,\nabla\Phi(s).
\]
Combining these yields $\alpha(s)\langle \nabla U(s),\Delta_s\rangle = -\langle \nabla\Psi(s),\Delta_s\rangle$.

\textbf{Step 2: Convert to a telescoping difference (Proof of Claim 1).}
By the second-order Taylor expansion of $\Psi$ around $s$ evaluated at $s'=s+\Delta_s$, there exists $\xi$ on the segment between $s$ and $s'$ such that:
\[
\Psi(s') = \Psi(s) + \langle \nabla\Psi(s),\Delta_s\rangle + \frac{1}{2}\Delta_s^\top \nabla^2\Psi(\xi)\Delta_s.
\]
Rearranging this isolates our term from Step 1:
\[
-\langle \nabla\Psi(s),\Delta_s\rangle = -(\Psi(s')-\Psi(s)) + \varepsilon(s,a,s')
\]
where $\varepsilon(s,a,s') := \frac{1}{2}\Delta_s^\top \nabla^2\Psi(\xi)\Delta_s$. Because $\mathcal{S}$ is compact and $\Psi \in C^2$, the operator norm of $\nabla^2\Psi$ is bounded by $L_\Psi$, yielding $|\varepsilon(s,a,s')|\le \frac{L_\Psi}{2}\|\Delta_s\|_2^2$.

\section{Experimental Setup and Hyperparameters}
\label{app:Hyperparams}

\rev{All environments and training routines are implemented in JAX and executed on
an NVIDIA H200 GPU; this hardware reflects available compute rather than a
requirement of the environment.}

\subsection{Uncertainty Oracle Construction}

For our synthetic continuous control experiments, we model the uncertainty oracle $U(s)$ analytically to simulate a localized region of high uncertainty (representing out-of-distribution states where the simulator is unreliable). Given a 2D state $s = (x, y)$, we define the uncertainty landscape using a Gaussian function: $U(s) = A \exp\left(-\frac{(x - c_x)^2 + (y - c_y)^2}{2\sigma^2}\right)$ where $A$ represents the maximum uncertainty amplitude, $(c_x, c_y)$ denotes the center coordinate of the uncertain region, and $\sigma$ controls the spatial spread of the uncertainty. 

\pagebreak
\subsection{Hyperparameters}

\begin{table}[h]
\centering
\caption{SAC and environment hyperparameters used in our experiments. All models were trained with a replay buffer capacity of $10\times 10^6$ transitions.}
\label{tab:hyperparameters}
\begin{tabular}{lll}
\toprule
\textbf{Category} & \textbf{Hyperparameter} & \textbf{Value} \\
\midrule

\multirow{7}{*}{Optimization}
& Actor Learning Rate & $1\times10^{-6}$ \\
& Critic Learning Rate & $5\times10^{-4}$ \\
& Alpha Learning Rate & $2\times10^{-3}$ \\
& Discount Factor ($\gamma$) & $0.99999999$ \\
& Soft Target Update Rate ($\tau$) & $10^{-4}$ \\
& Max Gradient Norm & $30.0$ \\

\midrule
\multirow{4}{*}{Training Loop}
& Batch Size & $256$ \\
& Replay Buffer Size & $10{,}000{,}000$ \\
& Action Limit & $0.1$ \\

\midrule
\multirow{4}{*}{Policy Architecture}
& Actor Type & Normalizing Flow \\
& Number of Flow Blocks & $32$ \\
& Hidden Dimension & $256$ \\
& Sequence Length & $6$ \\
\midrule
\multirow{3}{*}{Entropy}
& Target Entropy & $-8.0$ \\
& Initial Alpha ($\alpha$) & $0.1$ \\
& Initial Temperature & $1.0$ \\

\midrule
\multirow{7}{*}{Reward Shaping}
& Shaping Scale & $0.475$ \\
& Radial/Push Coefficient & $1.0$ \\
& Rotational/Swirl Coefficient & $5.0$ \\
& Swirl Clip & $5.0$ \\
& Smooth Swirl Clipping & True \\
& Step Reward Coefficient & $1.0$ \\
& Goal Reward Coefficient & $4.0$ \\


\bottomrule
\end{tabular}
\end{table}

\rev{
\paragraph{Time Conditioning.}
Because the desired behavior depends on the remaining episode horizon, we
condition both the actor and critic on normalized episode time. At each step
\(t\), we compute
\(
\rho_t = \frac{t}{H},
\)
where \(H=60\) is the episode horizon. The scalar \(\rho_t\) is encoded using a
sinusoidal time embedding and, in our final runs, is also concatenated directly
to the network input through the raw-time conditioning flag. For the
normalizing-flow actor, this time representation is concatenated with the
state-history  before parameterizing the
flow transformations. The critic receives the same time-conditioned input when
estimating \(Q(s_t,a_t,\rho_t)\). In the replay buffer, we store both
\(\rho_t\) and \(\rho_{t+1}\), so the SAC target uses the next-state time
\(\rho_{t+1}\) when evaluating the target policy and target critic. This
conditioning allows the same physical state to induce different actions at
different stages of the episode, enabling the policy to prioritize frontier
exploration early while preserving goal-directed behavior later in the horizon.}

\subsection{Continuous Box World Environment}

To evaluate our proposed vector-field reward shaping, we implemented a custom, GPU-accelerated 2D continuous navigation environment natively in JAX. The environment is designed to explicitly model localized areas of missing data and unreliable simulator dynamics, simulating the out-of-distribution (OOD) challenges inherent in offline reinforcement learning.

\textbf{State and Action Spaces:} The state space consists of continuous 2D coordinates bounded within $s \in [0, 1]^2$. The agent begins each episode at a uniformly sampled position in the bottom-left region $[0.05, 0.2]^\top$. The action space is a continuous 2D vector, $\Delta_s$, representing the agent's step, which is strictly clipped to $[-0.1, 0.1]^2$.

\textbf{Base Dynamics and Reward:} In the nominal (safe) regions, the environment transitions are governed by the agent's action combined with a small Gaussian background noise ($\mathcal{N}(0, 0.01)$). The primary task is to navigate to a goal region of radius $0.05$ centered at $(0.9, 0.9)$. The base dense reward is proportional to the reduction in distance to the goal (scaled by $10.0$), minus a constant step penalty of $0.01$. Successfully reaching the goal yields a $+20.0$ bonus and terminates the episode. The maximum episode horizon is $60$ steps.

\begin{revision}
\subsection{Additional Goal-Directed Exploration Diagnostics}
\label{app:goal_directed_diagnostics}
We provide additional training diagnostics for the goal-directed exploration
setting in Figures~\ref{fig:goal_exploration_metrics}
and~\ref{fig:goal_task_metrics}. Each curve is aggregated over five independent
random seeds; the solid line reports the mean and the shaded region shows the
minimum--maximum range across seeds. Figure~\ref{fig:goal_exploration_metrics} shows that the learned policy
increasingly allocates its state visitations to the target uncertainty
boundary. Boundary occupancy rises from approximately \(0.1\) early in training
to above \(0.75\) at the final checkpoint. The tangential speed also increases
and stabilizes at a positive value, showing that the agent continues moving
along the boundary rather than parking at a single state. Meanwhile, the
unsafe-transition rate steadily decreases and reaches zero across all five seeds.
Together, these results show that the policy learns sustained boundary
exploration without entering deeply uncertain regions. As shown in Figure~\ref{fig:goal_task_metrics}, the goal success rate increases
substantially during training and reaches nearly \(100\%\) across seeds. At the same time, the mean episode
length remains close to the \(60\)-step horizon and decreases only modestly at
later checkpoints. Because reaching the goal terminates the episode, this
combination of high success and relatively long episodes is consistent with
the intended behavior: the policy devotes a substantial portion of the horizon
to boundary exploration before moving toward and reaching the goal.

\begin{figure*}[t]
    \centering

    \begin{subfigure}[t]{0.31\textwidth}
        \centering
        \includegraphics[width=\linewidth]{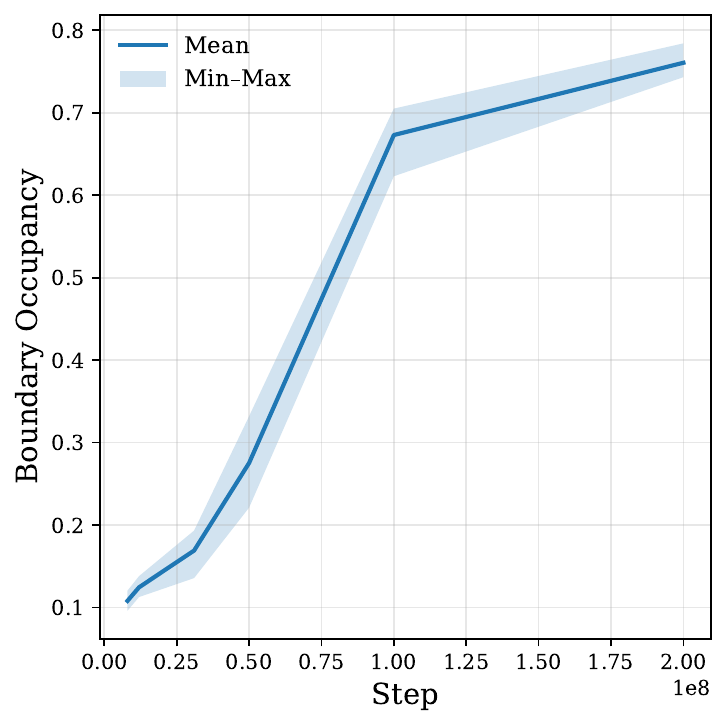}
        \caption{Boundary occupancy}
        \label{fig:goal_boundary_occupancy}
    \end{subfigure}
    \hfill
    \begin{subfigure}[t]{0.31\textwidth}
        \centering
        \includegraphics[width=\linewidth]{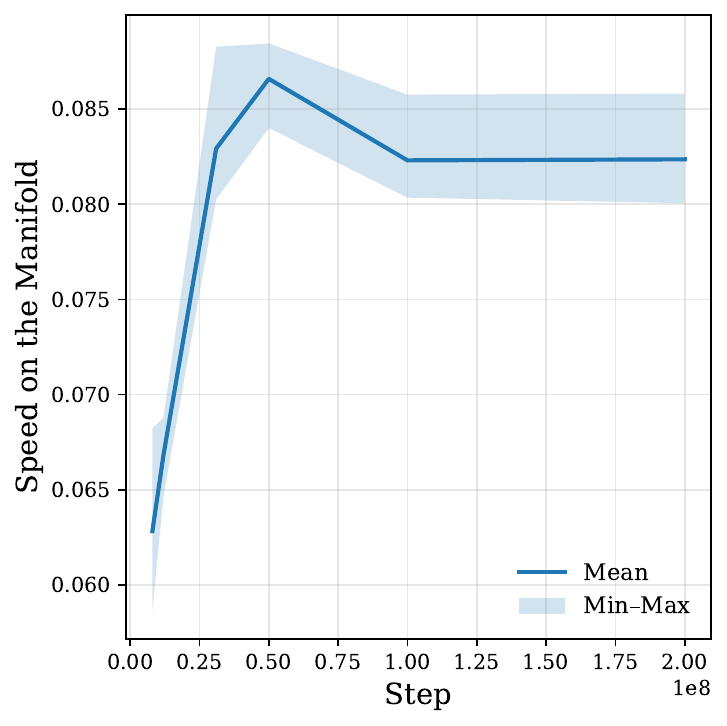}
        \caption{Tangential speed}
        \label{fig:goal_tangential_speed}
    \end{subfigure}
    \hfill
    \begin{subfigure}[t]{0.31\textwidth}
        \centering
        \includegraphics[width=\linewidth]
    {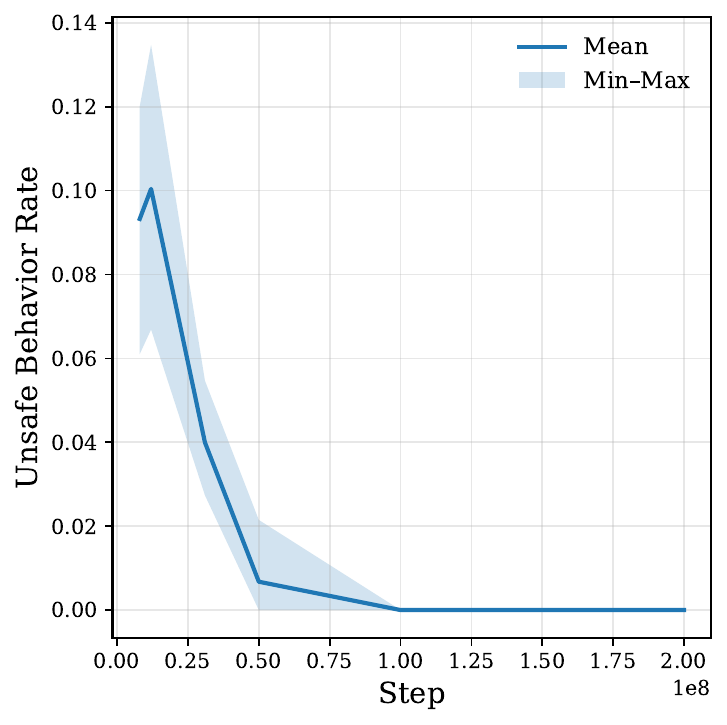}
        \caption{Unsafe-transition rate}
        \label{fig:goal_unsafe_rate}
    \end{subfigure}

    \caption{\textbf{Exploration and safety diagnostics in the goal-directed
    setting.} Results are aggregated over five random seeds, with the mean
    shown by the solid line and the minimum--maximum range shown by the shaded
    region. Boundary occupancy increases throughout training, while tangential
    speed rises and remains positive, indicating sustained motion along the
    target uncertainty boundary. At the same time, the unsafe-transition rate
    decreases to zero across seeds.}
    \label{fig:goal_exploration_metrics}
\end{figure*}

\begin{figure*}[h]
    \centering

    \begin{subfigure}[t]{0.45\textwidth}
        \centering
        \includegraphics[width=\linewidth]
        {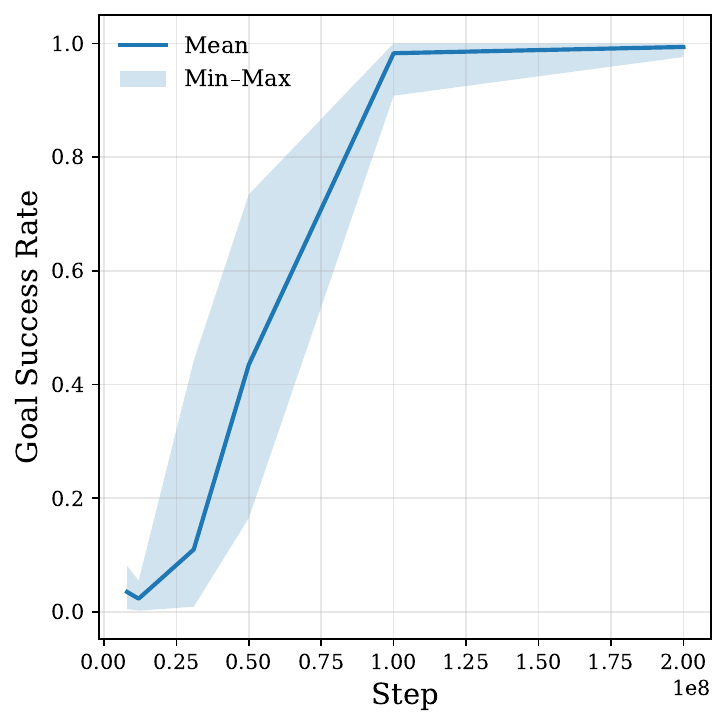}
        \caption{Goal success rate}
        \label{fig:goal_success_rate}
    \end{subfigure}
    \hfill
    \begin{subfigure}[t]{0.45\textwidth}
        \centering
        \includegraphics[width=\linewidth]{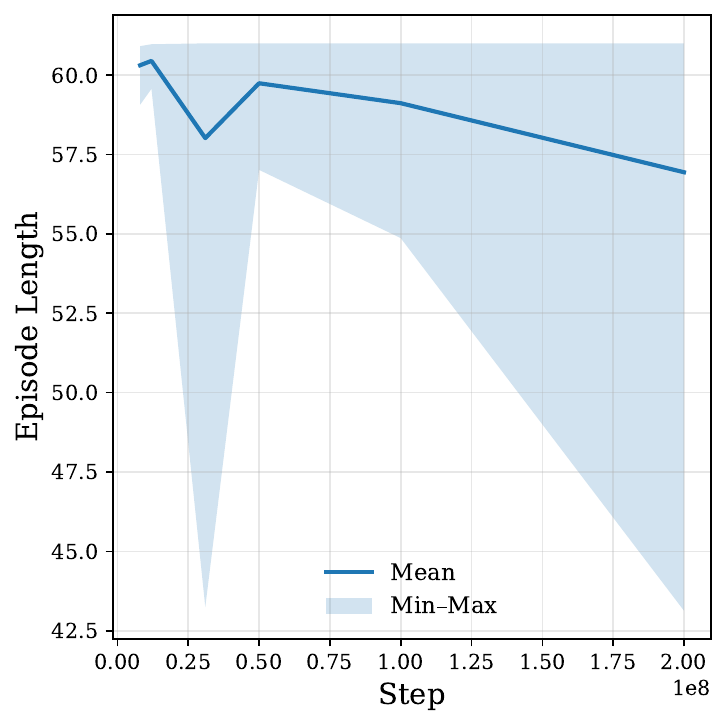}
        \caption{Episode length}
        \label{fig:goal_episode_length}
    \end{subfigure}

    \caption{\textbf{Primary-task performance in the goal-directed setting.}
    Results are aggregated over five random seeds. The goal success rate
    approaches one as training progresses. Episode length remains close to the
    finite-horizon limit, indicating that the learned policy uses much of the
    available horizon for boundary exploration before completing the primary
    task.}
    \label{fig:goal_task_metrics}
\end{figure*}
\end{revision}

\begin{revision}
\section{Multi-Seed Visitation Maps}
\label{app:additional_heatmaps}
\begin{figure*}[h]
    \centering

    \begin{subfigure}[t]{0.96\textwidth}
        \centering
        \includegraphics[width=\linewidth]{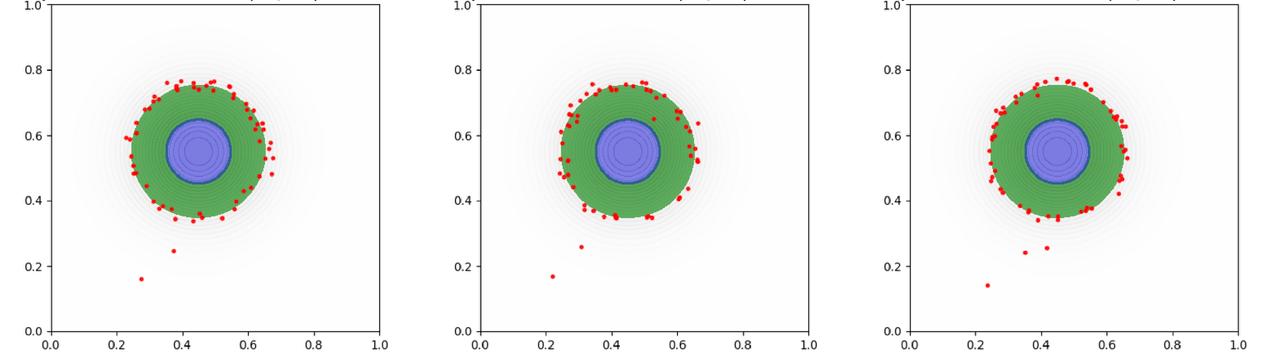}
        \caption{Proposed vector-field reward under pure exploration.}
        \label{fig:swirl_explore_seeds}
    \end{subfigure}

    \vspace{0.5em}

    \begin{subfigure}[t]{0.96\textwidth}
        \centering
        \includegraphics[width=\linewidth]{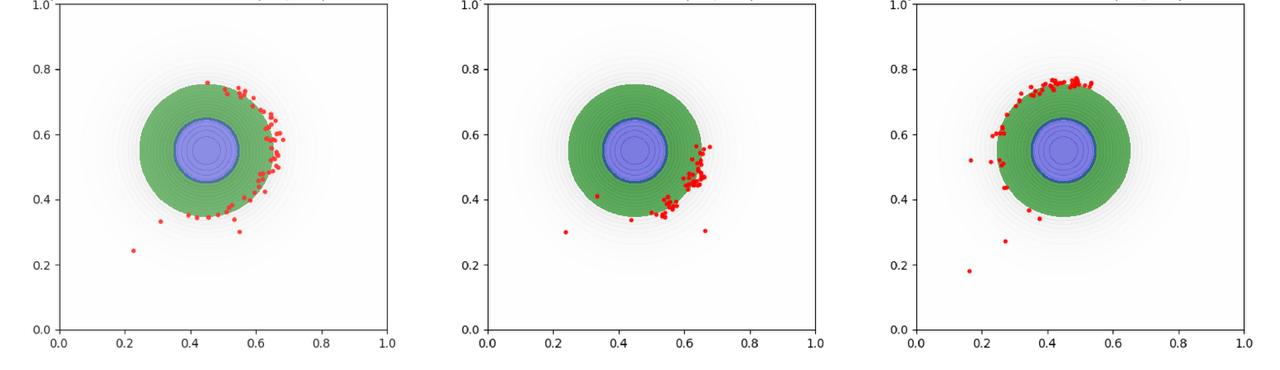}
        \caption{Static boundary-bump reward under pure exploration.}
        \label{fig:bump_explore_seeds}
    \end{subfigure}

    \vspace{0.5em}

    \begin{subfigure}[t]{0.96\textwidth}
        \centering
        \includegraphics[width=\linewidth]{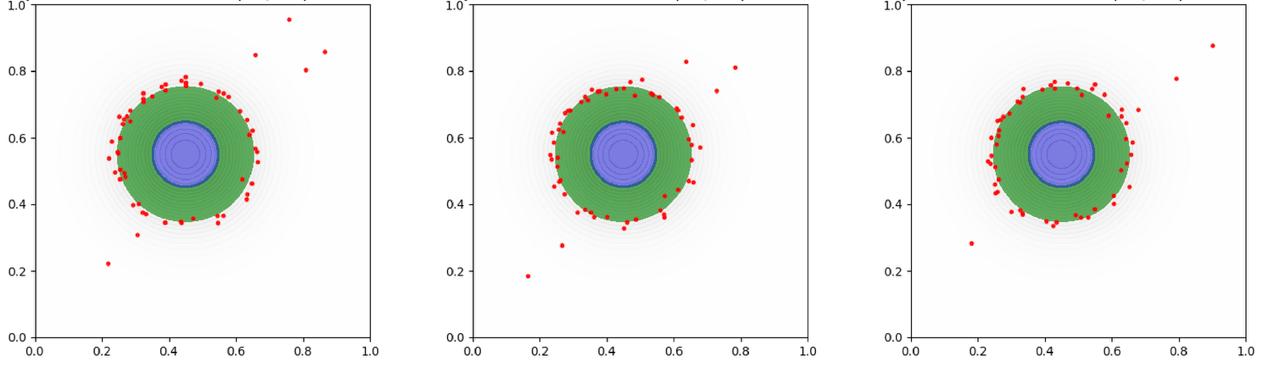}
        \caption{Proposed vector-field reward with the primary-task reward.}
        \label{fig:swirl_main_task_seeds}
    \end{subfigure}

    \caption{\textbf{State-visitation patterns across multiple random seeds.}
    Each row reports results from three independent seeds. In the
    pure-exploration setting, the proposed vector-field reward produces broad
    visitation around the target uncertainty boundary, whereas the static
    boundary-bump reward tends to concentrate visitation within a limited
    portion of the boundary. When the primary-task reward is restored, our
    method continues to explore substantial portions of the target boundary
    while also permitting the off-boundary movement required for task
    completion.}
    \label{fig:additional_seed_heatmaps}
\end{figure*}
Figure~\ref{fig:additional_seed_heatmaps} presents visitation maps from three
independent seeds for the pure-exploration and goal-directed settings. The results in Figure~\ref{fig:additional_seed_heatmaps} illustrate the
qualitative difference between rewarding proximity to the uncertainty boundary
and explicitly encouraging tangential motion along it. The boundary-bump reward
can attract the agent toward the desired uncertainty level, but it does not
provide an incentive to continue traversing the boundary, resulting in
random concentration around a small number of states. In contrast, the
rotational component of our reward promotes broader boundary coverage across
seeds. This behavior is largely preserved when the environment reward is
included, supporting the intended balance between informative sample
collection and primary-task completion.

\end{revision}

\section{Other Related Work}
\paragraph{Exploration in Reinforcement Learning.}
A broad range of approaches has been proposed to improve exploration in
reinforcement learning. These include contrastive-learning-based methods
\citep{mohamed2026temporal}, posterior-sampling and ensemble-based exploration approaches
\citep{osband2016deep, patil2026ogpo}, density-estimation-based exploration
\citep{bellemare2016unifying}, and methods that encourage the agent to reach
novel goals \citep{pong2019skew}. Another related direction is
Information-Directed Sampling, which balances exploration and exploitation by minimizing the ratio between squared expected regret and a measure of information gain
\citep{russo2014learning}.

\paragraph{Conditioning.}
We use time conditioning so that the same physical state can induce different
actions at different stages of an episode, allowing the policy to prioritize
frontier exploration early while preserving goal-directed behavior later in
the horizon. Conditioning is widely used in reinforcement learning
\citep{myers2026offline,chen2021decision,kumar2019reward}, as well as in more
complex settings such as language-model alignment. Following the perspective
of \citet{nourzad2026one,nourzadbarriers,wang2024conditional}, our setting can
also be viewed as a multi-task learning problem in which frontier exploration
and goal completion constitute two competing tasks. Similar to
\citet{nourzadbarriers}, conditioning enables a single policy to represent and
coordinate both behaviors.

\paragraph{Reinforcement Learning in Real-World Systems.}
Reinforcement learning has been applied to a wide range of real-world systems,
including language-model alignment
\citep{setlur2025e3, nourzad2025firm, rafailov2023direct, ouyang2022training}, autonomous driving
\citep{wang2025alpamayo,lu2023imitation}, and communication networks
\citep{alemdar2025remarkable,roknilamouki2025safe}. Our problem formulation,
which seeks to train a policy that maintains strong task performance while
collecting informative new samples, is  motivated by these domains, as
well as by other safety-sensitive applications such as healthcare\citep{brunskill2010policies, gottesman2020interpretable, amirzade6030962mechanistic, dann2019policy} and energy management
\citep{glavic2017reinforcement, shi2021competitive,fathizadan2026energy}.




\end{document}